\RequirePackage{fix-cm}
\documentclass{svjour3}                     
\smartqed  
\usepackage{graphicx}

\usepackage{algorithmicx}
\usepackage{algpseudocode}
\usepackage[T1]{fontenc}
\usepackage{parskip}
\usepackage{amsmath,amssymb}
\usepackage{stmaryrd}
\usepackage{float}
\usepackage{graphicx}
\usepackage{xcolor}

\newcommand{\OT}{\mathrm{OT}}
\newcommand{\ADM}{\mathrm{ADM}}
\newcommand{\Tr}{\mathrm{Tr}}
\newcommand{\E}{\mathbb{E}}
\newcommand{\Probs}{\mathcal{P}}
\newcommand{\N}{\mathcal{N}}
\newcommand{\KL}{D_{\mathrm{KL}}}
\newcommand{\trace}{\mathrm{Tr}}
\newcommand{\Log}{\mathrm{Log}}

\renewcommand{\H}{\mathcal{H}}

\newcommand{\Ncal}{\mathcal{N}}

\newcommand{\X}{\mathcal{X}}
\newcommand{\R}{\mathbb{R}}

\newcommand{ \ep}{\varepsilon}
\newcommand{ \Lexp}{L^{\rm exp}_{\ep}}

\newcommand{ \Dep}{D_{\ep}}

\newcommand{ \OTep}{\operatorname{OT}_{\ep}}

\newcommand{ \gammaep}{\gamma^{\ep}}

\newcommand{\Pro}{\mathcal{P}}

\newcommand{\restr}[1]{\lower3pt\hbox{$|_{#1}$}}

\newcommand{\Fcep}{\mathcal{F}^{(c, \ep)}}

\newcommand{\ucep}{u^{(c,\ep)}}
\newcommand{\wcep}{w^{(c,\ep)}}

\DeclareMathOperator*{\argmin}{arg\,min}

\newtheorem{deff}{Definition}
\begin{document}
\title{Entropy-Regularized $2$-Wasserstein Distance between Gaussian Measures}


\author{Anton Mallasto$^*$ \and
        Augusto Gerolin  \and
        H\`a Quang Minh
}


\institute{A. Mallasto \at
              Aalto University, Department of Computer Science \\
              \email{anton.mallasto@aalto.fi}            \\
             $^*$Corresponding author
           \and
           A. Gerolin \at
              Vrije Universiteit Amsterdam, Department of Theoretical Chemistry \\
              \email{augustogerolin@gmail.com}
            \and 
            H.Q. Minh \at
            RIKEN Center for Advanced Intelligence Project \\
         \email{minh.haquang@riken.jp}
}

\date{}

\maketitle

\begin{abstract}
Gaussian distributions are plentiful in applications dealing in uncertainty quantification and diffusivity. They furthermore stand as important special cases for frameworks providing geometries for probability measures, as the resulting geometry on Gaussians is often expressible in closed-form under the frameworks. In this work, we study the Gaussian geometry under the entropy-regularized $2$-Wasserstein distance, by providing closed-form solutions for the distance and interpolations between elements. Furthermore, we provide a fixed-point characterization of a population barycenter when restricted to the manifold of Gaussians, which allows computations through the fixed-point iteration algorithm. As a consequence, the results yield closed-form expressions for the $2$-Sinkhorn divergence. As the geometries change by varying the regularization magnitude, we study the limiting cases of vanishing and infinite magnitudes, reconfirming well-known results on the limits of the Sinkhorn divergence. Finally, we illustrate the resulting geometries with a numerical study. 

\end{abstract}

\section{Introduction}
\emph{Optimal transport} (OT)~\cite{villani08} studies the geometry of probability measures through the lifting of a cost function between samples.  This is carried out by devising a coupling between two probability measures via a \emph{transport plan}, so that one measure is transported to another with minimal total cost. The resulting geometry offers a favorable way of comparing probability measures one to another, which has lead to considerable success in machine learning, especially in generative modelling~\cite{arjovsky17,deshpande18,pmlr-v97-dukler19a,mallasto19}, where one aims at training a model distribution to sample from a given data distribution, and computer vision, where OT provides intuitive metrics between images~\cite{rubner00}. Notably, OT can not only be used to derive divergences, but also metrics between probability distributions, referred to as the \emph{$p$-Wasserstein metrics}.

To ease the computational aspects of OT, entropic relaxation was introduced, which transforms the constrained convex problem of transportation into an unconstrained strictly convex problem~\cite{cuturi13}. This is carried out via considering the sum of the total cost and the Kullbackk-Leibler (KL) divergence, between the transport plan and the independent joint distribution, scaled by some regularization magnitude. In addition to computational aspects, the entropic regularization also betters statistical properties~\cite{Sommerfeld2017WassersteinDO}, specifically, the complexity of estimating the OT quantity between measures through sampling~\cite{genevay18sample,MenWee19,weed17}. Theoretical properties of the entropic regularization have been studied in e.g. metric geometry, machine learning and statistics~\cite{feydy18,genevay16,genevay17,GigTam18,MalMonGer19,ramdas2017,RipThesis}. It has also been applied in a variety of fields, including computer vision, density functional theory in chemistry, and inverse problems (e.g.~\cite{genevay17,GerGroGor19,Lunz18,patrini18}). 

The resulting problem has close relations to the \emph{Schr\"odinger problem}~\cite{Schr31}, which considers the most likely flow of a cloud of gas from an initial position to an observed position after a certain amount of time under a prior assumption on the evolution of the position, given by e.g. a Brownian motion. The resulting problem has found applications in fields such as mathematical physics, economics, optimization and probability~\cite{BorLewNus94,Csi75,peyre17,FraLor89,GalichonEconomics,RusIPFP,Zam15}). Connections to OT have been considered in e.g. \cite{cuturi13,galsal,LeoSurvey,rus93,rus98}.

OT is not the only instance of a geometric framework for probability measures. Other popular choices include \emph{information geometric divergences}~\cite{amari16,ay17} and \emph{integral probability metrics}~\cite{muller97}. In contrast to these methods, OT and entropic OT has the advantage of metrizing the weak$^*$-convergence of probability measures, which results in non-singular behavior when comparing measures of disjoint supports. On top of this, being able to decide the lifted cost function is important in applications, as the cost function can be used to incorporate modelling choices, determining which differences in samples are deemed most important. For example, the standard Euclidean metric is a poor choice for comparing images.

Gaussian distributions provide a meaningful testing ground for such frameworks since, in many cases, they result in closed-form expressions. In addition, the study of Gaussians under the OT framework result in useful divergences. In particular, divergences between centered Gaussians result in divergences between their corresponding covariance matrices. Both instances enjoy many applications in a plethora of fields, such as medical imaging~\cite{Dryden:2009}, computer vision~\cite{tuzel2006region,tuzel08,tuzel07}, brain computer interfaces~\cite{Congedo:BCIreview2017}, natural language processing~\cite{muzellec2018}, and assessing the quality of generative models~\cite{heusel17}. Notably, the $2$-Wasserstein metric between Gaussians  is known as the \emph{Bures metric} in quantum physics, where it is used to compare quantum states. Other popular divergences for Gaussians include the \emph{affine-invariant Riemannian metric} \cite{Pennec:IJCV2006}, corresponding to the Fisher-Rao distance between centered Gaussians, the Alpha Log-Determinant divergences \cite{Chebbi:2012Means}, corresponding to R\'enyi divergences between centered Gaussians, and the \emph{log-Euclidean metric} \cite{LogEuclidean:SIAM2007}. A survey of some of the most common divergences and their resulting geometry on Gaussians can be found in~\cite{feragen17}. More recently, applications have driven research into allowing determining optimal divergences for the task at hand, which has raised interest in studying interpolations between different divergences~\cite{amari2018information,cichocki15,thanwerdas19}. Generalizations of these divergences to the infinite-dimensional setting of Gaussian processes and covariance operators have also been considered \cite{Larotonda:2007,mallasto17,Masarotto:2018Procrustes,MinhSB:NIPS2014,Minh:LogDet2016}. 

The \emph{Sinkhorn divergence} has been proposed in OT, applying the entropic regularization to define a parametric family of divergences, interpolating from the OT quantity to a \emph{maximum mean discrepancy} (MMD), whose kernel is determined by the cost. In the present work, we provide a closed-form solution to the entropy-regularized $2$-Wasserstein distance between multivariate Gaussians, which can then be applied in the computation of the corresponding Sinkhorn divergence between Gaussians. In addition, we study the task of interpolating between two Gaussians under the entropy-regularized $2$-Wasserstein distance, and confirm known limiting properties of the divergences with respect to the regularization strength. Finally, we provide fixed-point expressions for the barycenter of population of Gaussians restricted to the Gaussian manifold, that can be employed in fixed-point iteration for computing the barycenter. The one-dimensional setting has been studied in~\cite{amari2018information,gentil2017,GerGroGor19}. The Schr\"odinger bridge between multivariate Gaussians has been considered in~\cite{chen15}, including the study of the limiting case of bringing the noise of the driving Brownian motion to $0$, resulting in the $2$-Wasserstein case, in~\cite{chen16}.

During the review process of the article at hand, analogous results of this paper was obtained independently by Janati et al. \cite{janati20}. In the barycenter problem, in \cite{janati20} the authors shows that the barycenter of Gaussians under the Sinkhorn divergence is a Gaussian, when restricted to the space of sub-Gaussian measures. This extends our Theorem \ref{thm:sinkhornbarycenter}, where we explicitly restrict to Gaussian instead of sub-Gaussian measures. Furthermore, the authors consider the setting of unbalanced Gaussian measures."

The paper is divided as follows: in Section~\ref{sec:back}, we briefly introduce the necessary background to develop the entropic OT theory of Gaussians, including the formulation of OT, entropic OT, and the corresponding dual and dynamical formulations. In Section~\ref{sec:sink}, we compute explicit solutions to the entropy-relaxed $2$-Wasserstein distance between Gaussians, including the dynamical formulation that allows for interpolation. As a consequence, we derive a closed-form solution for the corresponding Sinkhorn divergence. In Section~\ref{sec:barycenter}, we study the \emph{barycenters} of populations of Gaussians, restricted to the Gaussian manifold. We derive fixed-point expressions for the entropic $2$-Wasserstein distance and the $2$-Sinkhorn divergence. Finally, in Section~\ref{sec:numerics}, we illustrate the resulting interpolative and barycentric schemes. Especially, we consider varying the regularization magnitude, visualizing the interpolation between the OT and MMD problems in the Sinkhorn case~\cite{feydy18,genevay17,ramdas2017}.

\section{Background}\label{sec:back}
In this section, we start by recalling the essential background for optimal transport (OT) and its entropy-relaxed version. More in-depth exposition for OT can be found in~\cite{villani08}, and for computational aspects and entropic OT in~\cite{peyre17}.

\textbf{Optimal transport.} Let $(\X,d)$ be a metric space equipped with a lower semi-continuous \emph{cost function} $c:\X\times \X \to \mathbb{R}_{\geq 0}$. Then, the optimal transport problem between two probability measures $\mu, \nu \in \Probs(\X)$ is given by
\begin{equation}\label{eq:OT}
    \OT(\mu, \nu) = \min_{\gamma\in \ADM(\mu,\nu)}\E_\gamma[c],
\end{equation}
where $\ADM(\mu,\nu)$ is the set of joint probabilities with marginals $\mu$ and $\nu$, and $\E_\mu[f]$ denotes the expected value of $f$ under $\mu$
\begin{equation}
\E_\mu[f] = \int_\X f(x) {\rm d}\mu(x).    
\end{equation}
Additionally, by $\E[\mu]$ we denote the expectation of $\mu$. A minimizer of \eqref{eq:OT} is denoted by $\gamma_{\rm opt}$ and called a \emph{transport plan}.

The OT problem admits the following \emph{Kantorovich (dual) formulation}
\begin{equation}\label{eq:ot_dual}
    \OT(\mu,\nu) = \max\limits_{\varphi, \psi \in \ADM(c)} \left\lbrace
    \E_\mu[\varphi] + \E_\nu[\psi]
    \right\rbrace,
\end{equation}
where $(\varphi, \psi) \in \ADM(c)$ is required to satisfy
\begin{equation}
    \varphi(x) + \psi(y) \leq c(x,y),\quad \forall(x,y) \in \X\times \X.
\end{equation}
Potentials  $\varphi_{\rm{opt}}, \psi_{\rm{opt}}$ achieving the maximum in \eqref{eq:ot_dual} are called \emph{Kantorovich potentials}.

\textbf{Wasserstein distances.} The $p$-Wasserstein distance $W_p$ between $\mu$ and $\nu$ is defined as
\begin{equation}
    W_p(\mu,\nu) = \OT_{d^p}(\mu, \nu)^{\frac{1}{p}},
\end{equation}
where $d$ is a metric on $X$ and $p\geq 1$. The case $p=2$ is particularly interesting, as the resulting metric is then induced by a pseudo-Riemannian metric structure \cite{AmGiSa,malago18}. 

\textbf{$2$-Wasserstein distance between Gaussians.} One of the rare cases where the $2$-Wasserstein distance admits a closed form solution is between two multivariate Gaussian distributions $\mu_i=\N(m_i,K_i)$, $i=0,1$ with $d(x,y) = \|x-y\|$, which is given by \cite{dowson82,givens84,knott84,olkin82}
\begin{equation}\label{eq:gaus_was_d}
    W_2^2(\mu_0, \mu_1) = ||m_0 -m_1||^2 + \Tr(K_0) + \Tr(K_1) - 2 \Tr\left(K_1^\frac{1}{2} K_0 K_1^\frac{1}{2}\right)^\frac{1}{2}.
\end{equation}
It can be shown that \eqref{eq:gaus_was_d} is induced by a \emph{Riemannian metric} in the space of $n$-dimensional Gaussians $\Ncal(\R^n)$,  with the metric $g_K:T_K\Ncal(\R^n)\times T_K\Ncal(\R^n)\to \R$ given by~\cite{takatsu11}
\begin{equation}
    g_K(U,V) = \Tr\left[v_{(K,U)}Kv_{(K,V)}\right], \quad \forall ~K\in \Ncal(\R^n),~ U,V \in T_K\Ncal(\R^n),
\end{equation}
where $v_{(K,V)}$ denotes the unique symmetric matrix solving the \emph{Sylvester equation}
\begin{equation}
V = Kv_{(K,V)} + v_{(K,V)}K.
\end{equation}

Moreover, given $\Ncal(m_0,K_0),\Ncal(m_1,K_1)\in \Ncal(\R^n)$, the geodesics under the metric \eqref{eq:gaus_was_d} are given by $\Ncal(m_t, K_t)$, with~\cite{mccann97}
\begin{equation}\label{eq:gaus_was_geodesic}
\begin{aligned}
    m_t =& (1-t)m_0 + tm_1, \\
    K_t =& \left((1-t)I + tK_0^{-\frac{1}{2}}\left(K_0^\frac{1}{2}K_1K_0^\frac{1}{2}\right)^\frac{1}{2}K_0^{-\frac{1}{2}}\right)K_0\\
    &\times\left((1-t)I + tK_0^{-\frac{1}{2}}\left(K_0^\frac{1}{2}K_1K_0^\frac{1}{2}\right)^\frac{1}{2}K_0^{-\frac{1}{2}}\right)\\
    &= (1-t)^2K_0 + t^2K_1 + t(1-t)[(K_0K_1)^{1/2} + (K_1K_0)^{1/2}].
\end{aligned}
\end{equation}
We remark that Eq.(\ref{eq:gaus_was_d}) is valid for all Gaussian distributions, including the case when $K_0, K_1$ are positive semi-definite. This is in contrast to the affine-invariant Riemannian distance $||\log(K_0^{-1/2}K_1K_0^{-1/2})||_F$, the Log-Euclidean distance
$||\log(K_0)-\log(K_1)||_F$, and the Kullback-Leibler divergence (see below), which require that $K_0,K_1$ be strictly positive definite.

Finally, the $2$-Wasserstein barycenter $\bar{\mu}$ of a population of probability measures $\mu_i$ with weights $\lambda_i\geq 0$, $i=1,2,..,N$ and $\sum_{i=1}^N\lambda_i = 1$, is defined as the minimizer
\begin{equation}
    \bar{\mu} := \argmin\limits_{\mu \in \Probs(\R^n)} \sum_{i=1}^N \lambda_i W_2^2(\mu,\mu_i).
\end{equation}
When the population consists of Gaussians $\mu_i = \Ncal(m_i, K_i)$, one can show that the barycenter is Gaussian given by $\bar{\mu} = \Ncal(\bar{m}, \bar{K})$, where $\bar{m}$, $\bar{K}$ satisfy~\cite[Thm. 6.1]{agueh11}
\begin{equation}\label{eq:wasserstein_barycenter}
    \bar{m} = \sum_{i=1}^N \lambda_i m_i,\quad
    \bar{K} = \sum_{i=1}^N \lambda_i \left(K^\frac{1}{2}K_iK^\frac{1}{2}\right)^\frac{1}{2}.
\end{equation}

\textbf{Entropic relaxation.}
Let $\mu, \nu \in \Probs(X)$ with densities $p_\mu$ and $p_\nu$. Then, we denote by
\begin{equation}
    \KL(\mu || \nu) = - \E_\mu\left[\log\frac{p_\nu}{p_\mu}\right],
\end{equation}
the \emph{Kullback-Leibler divergence} (KL-divergence) between $\mu$ and $\nu$. The \emph{differential entropy} of $\mu$ is given by
\begin{equation}
    H(\mu) = -\E_\mu[\log p_\mu]. 
\end{equation}
For a product measure, we have the identity
\begin{equation}\label{eq:kl_identity}
    \KL(\gamma || \mu_0 \otimes \mu_1) = H(\mu_0) + H(\mu_1) - H(\gamma).
\end{equation}

A special case that will be used later in this work is the KL-divergence between two non-degenerate multivariate Gaussian distributions $\mu_0 = \N(m_0, K_0)$ and $\mu_1 = \N(m_1, K_1)$ when $X = \mathbb{R}^n$, which is given by
\begin{equation}
\begin{aligned}
    \KL(\mu || \nu) =& \frac{1}{2}\left(
    \Tr\left(K_0^{-1}K_1\right)
    + \left(m_1 - m_0\right)^T K_0^{-1} \left(m_1 - m_0\right) \phantom{\frac{K}{K}}\right.\\
    &\left. \phantom{\frac{K}{K}}- n + \ln \left( \frac{\det K_1}{\det K_0}\right)
    \right),
\end{aligned}
\end{equation}
and for the entropy we have
\begin{equation}
    H(\mu_0) = \frac{1}{2}\log \det \left(2\pi e K_0\right).
\end{equation}

Given $\epsilon > 0$, we relax \eqref{eq:OT} with a KL-divergence term between the transport plan and the independent joint distribution as, yielding the \emph{entropic OT problem}~\cite{cuturi13}
\begin{equation}\label{eq:mainKL}
    \OT_c^\epsilon(\mu, \nu) = \min_{\gamma\in \ADM(\mu,\nu)}\left\lbrace\E_\gamma[c]
    + \epsilon \KL(\gamma || \mu \otimes \nu) \right\rbrace,
\end{equation}
which yields a strictly convex problem with respect to $\gamma$. Moreover, this problem is numerically more favorable to solve \eqref{eq:OT} compared, for instance, to the \emph{Hungarian} and the \emph{auction algorithm}, due to the Sinkhorn-Knopp algorithm. As shown, for instance in  \cite{BorLewNus94,Csi75,DMaGer19,GigTam18,RusIPFP}, the above problem has a unique minimizer given by
\begin{equation}
\gammaep = \alpha^{\ep}(x)\beta^{\ep}(y)k(x,y)\mu(x)\nu(y),
\end{equation}
if and only if there exists functions $\alpha^\ep$ and $\beta^\ep$ such that
\begin{equation}\label{intro:SchSys}
\begin{aligned}
     \alpha^{\ep}(x)\E_\nu\left[\beta^{\ep}k(x,\cdot)\right] &= 1, \\
     \beta^{\ep}(y)\E_\mu\left[\alpha^{\ep}k(\cdot,y)\right] &= 1,\\
\end{aligned}
\end{equation}
where $k(x,y) = \exp\left(-\frac{1}{\epsilon}c\right)$ denotes the \emph{Gibbs kernel}. We call $\gamma^\epsilon$ an \emph{entropic transport plan}. Moreover, when $\ep\to 0$, $\gammaep$ converges to $\gamma_{\rm opt}$, a solution of the OT problem \eqref{eq:OT} \cite{peyre17,GerKauRajEnt,LeoSurvey}; while when $\ep\to\infty, \gammaep$ converges to the independent coupling $\gamma^\infty = \mu\otimes\nu$ \cite{genevay17,ramdas2017}. The latter property shows in particular that, for large $\ep$, the entropy-Regularized OT behaves like an inner product and not like a norm. In linear algebra, the polarization formula is the usual way of defining a norm from a inner product. That is the main idea of Sinkhorn divergence.

\textbf{Sinkhorn divergence.} The KL-divergence term in $\OT_c^\epsilon$ acts as a bias, as discussed in \cite{feydy18}. This can be removed by defining the \emph{p-Sinkhorn divergence} as
\begin{equation}\label{def:sinkhorn}
    S_p^\epsilon(\mu, \nu) = \OT_{d^p}^\epsilon(\mu, \nu) - \frac{1}{2}(\OT_{d^p}^\epsilon(\mu,\mu) + \OT_{d^p}^\epsilon(\nu,\nu) ).
\end{equation}

As shown in \cite{feydy18} if, for example, $c=d^p, p\geq 1$ the Sinkhorn divergences metrizes the convergence in law in the space of probability measures.

\textbf{Entropy-Kantorovich duality.} In this subsection we summarize well-known results on the Entropy-Kantorich. For further details and proofs, we refer the reader to \cite{DMaGer19}.

Given a probability measure $\mu$, the class of Entropy-Kantorovich potentials is defined by the set of measurable functions $\varphi$ on $\R^n$ satisfying
\begin{equation}
\Lexp(\R^n,\mu) = \left\lbrace \varphi:\R^n \to [-\infty, \infty[ \, : \,
0<\E_\mu\left[\exp\left(\frac{1}{\epsilon} \varphi\right)\right] < \infty
\right\rbrace.
\end{equation}

Then, given $c=d^2$, where $d(x,y) = \|x-y\|$, $\varphi\in \Lexp(\R^n,\mu_0)$ and $\psi\in \Lexp(\R^n,\mu_1)$, the \emph{entropic Kantorovich (dual) formulation} of $\OT^\epsilon_{d^2}(\mu,\nu)$ is given by
\cite{DMaGer19,feydy18,genevay17,GigTamBB18,LeoSurvey},
\begin{equation}\label{kanto}
\begin{aligned}
\OT^\epsilon_{d^2}(\mu_0,\mu_1) =& \sup\limits_{\varphi, \psi}\left\lbrace\E_{\mu_0}[\varphi] + \E_{\mu_1}[\psi]\phantom{\frac{K}{K}}\right.\\
&\left.-\ep \left(\E_{\mu_0 \otimes \mu_1} \left[\exp\left(\frac{(\varphi\oplus \psi)-d^2}{\ep}\right)\right]-1\right)\right\rbrace,
\end{aligned}
\end{equation}
where $\left(\varphi \oplus \psi\right)(x,y) = \varphi(x) + \psi(y)$, $\varphi\in \Lexp(\R^n,\mu_0)$, and $\psi\in \Lexp(\R^n,\mu_1)$.

Finally, we are able to state the full duality theorem between the primal \eqref{eq:mainKL} and the dual problem \eqref{kanto}. The Theorem below is a particular case of Theorem 2.8 and Proposition 2.11 in \cite{DMaGer19}, when we are in the Euclidian space with distance square cost function. \medskip

\begin{theorem}\label{thm:equiv_comp}
Let $\ep>0$ be a positive number, $c=d^2$, $\mu_0,\mu_1 \in \Pro(\R^n)$ be probability measures. Then, the supremum in \eqref{kanto} is attained for a unique couple $(\varphi^\epsilon, \psi^\epsilon)$ (up to the trivial transformation $(\varphi^\epsilon, \psi^\epsilon) \to (\varphi^\epsilon + \alpha, \psi^\epsilon - \alpha)$). Moreover, the following are equivalent:
\begin{itemize}
\item[\textbf{a.}] \emph{(Maximizers)} $\varphi^\epsilon$ and $\psi^\epsilon$ are maximizing potentials for \eqref{kanto}.

\item[\textbf{b.}] \emph{(Schr\"{o}dinger system)} Let
\begin{equation}
\gamma^\epsilon=\exp\left(\frac{1}{\epsilon}\left(\varphi^\epsilon \oplus \psi^\epsilon-d^2\right)\right)\mu_0\otimes \mu_1,    
\end{equation}
then $\gamma^\epsilon \in \ADM(\mu_0, \mu_1)$. Furthermore, $\gamma^\epsilon$ is the (unique) minimizer of the problem \eqref{eq:mainKL}.
\end{itemize}
\end{theorem}

Elements of the pair $(\varphi^\epsilon, \psi^\epsilon)$ reaching a maximum in \eqref{kanto} are called \emph{entropic Kantorovich potentials}. Finally, a relationship between $\alpha^\epsilon,\beta^\epsilon$ in \eqref{intro:SchSys}, and the entropic Kantorovich potentials $\varphi^\epsilon, \psi^\epsilon$ above, is according to Theorem~\ref{thm:equiv_comp} given by
\begin{equation}
    \varphi^\epsilon = \epsilon \log \alpha^\epsilon,\quad \psi^\epsilon = \epsilon \log \beta^\epsilon.
\end{equation}

Using the dual formulation, we can show the following.

\begin{proposition}\label{prop:convex_entropic_transport}
Let $\mu,\nu\in \Probs(\R^n)$. Then, $\OT^\epsilon_c(\mu,\nu)$ is strictly convex in both arguments.
\end{proposition}
\begin{proof}
Let $\mu_t$ = $t\mu_0 + (1-t)\mu_1$, and $(\varphi_j,\psi_j)$ be the entropic Kantorovich potentials associated with $\OT_c^\epsilon(\mu_j, \nu)$ for $j=0,1$, and $(\varphi, \psi)$ for $\OT_c^\epsilon(\mu_t, \nu)$. Then, using the dual formulation~\eqref{kanto}, we have
\begin{equation}
\begin{aligned}
\OT_c^\epsilon(\mu_t,\nu)=& t\left(\E_{\mu_0}[\varphi] + \E_{\nu}[\psi]\right)
+(1-t)\left(\E_{\mu_1}[\varphi] + \E_{\nu}[\psi]\right)\\
&- \epsilon t \left(\E_{\mu_0 \otimes \nu}\left[
\exp\left(\frac{(\varphi \otimes \psi)-c}{\epsilon}\right)\right]-1\right)\\
&- \epsilon (1-t)\left(\E_{\mu_1 \otimes \nu}\left[
\exp\left(\frac{(\varphi \otimes \psi)-c}{\epsilon}\right)\right]-1\right)\\
<&t\left(\E_{\mu_0}[\varphi_0] + \E_{\nu}[\psi_0]\right)
+(1-t)\left(\E_{\mu_1}[\varphi_1] + \E_{\nu}[\psi_1]\right)\\
&- \epsilon t \left(\E_{\mu_0 \otimes \nu}\left[
\exp\left(\frac{(\varphi_0 \otimes \psi_0)-c}{\epsilon}\right)\right]-1\right)\\
&- \epsilon (1-t)\left(\E_{\mu_1 \otimes \nu}\left[
\exp\left(\frac{(\varphi_1 \otimes \psi_1)-c}{\epsilon}\right)\right]-1\right)\\
=& t\OT_c^\epsilon(\mu_0,\nu) + (1-t)\OT_c^\epsilon(\mu_1,\nu),
\end{aligned}
\end{equation}
where the first equality results from linearity of expectations, and the inequality from noticing that the pair $(\varphi, \psi)$ is a competior for $(\varphi_j, \psi_j)$, $j=0,1$, but due to uniqueness of the entropic Kantorovich potentials (up to scalar additives, Theorem~\ref{thm:equiv_comp}), $(\varphi,\psi)$ cannot be equal to $(\varphi_0,\psi_0)$ and $(\varphi_1,\psi_1)$ (unless $\mu_0 = \mu_1)$, and will thus return lower values.
\qed
\end{proof}

\textbf{Dynamical formulation of entropy relaxed optimal transport.} Analogously to unregularized OT theory, the entropic-regularization of OT with distance cost admits a \emph{dynamical} (aka \emph{Benamou-Brenier}) formulation. 

In the following, we again consider the particular case when the cost function is given by $c(x,y) = \| x-y\|^2$. Then, we can write \eqref{eq:mainKL} as~\cite{GigTamBB18,LeoSurvey}
\begin{equation}
\OT^\epsilon_{d^2}(\mu_0,\mu_1) = \min_{(\mu^\epsilon_t,v_t)}\int_0^1\E_{\mu^\epsilon_t}\left[\|v_t\|^2\right] dt + H(\mu_0) + H(\mu_1),
\end{equation}
where $t\in[0,1]$, $\mu^\epsilon_0=\mu_0$, $\mu^\epsilon_1 = \mu_1$, and 
\begin{equation}
\partial_t \mu^\epsilon_t + \nabla\cdot (v_t \mu^\epsilon_t) = \frac{\ep}{2}\Delta \mu^\epsilon_t.    
\end{equation}

where the minimum must be understood as taken among all couples $(\mu^\epsilon_t,v_t)$ solving the continuity equation in the distributional sense (see appendix A); moreover, the minimum is attained if and only if $(\mu^\epsilon_t,v_t) = (\mu^\epsilon_t,\nabla\phi_t^\ep)$, for a potential $\phi_t^\ep:\R^d\to\R$, which is defined in the following via the entropic potentials. The resulting $\mu_t^\epsilon$ is called the \emph{entropic interpolation} between $\mu_0$ and $\mu_1$.

The solution can be characterized by (while abusing the notation and writing $\mu(x)$ for the density of $\mu$, which will be done throughout this work)
\begin{equation}
\gammaep(x,y) = \alpha^{\ep}(x)\beta^{\ep}(y) \exp\left(-\frac{1}{\epsilon}\|x-y\|^2\right) \mu_0(x)\mu_1(y),
\end{equation}
 in \eqref{intro:SchSys} of the static problem \eqref{eq:mainKL} in conjunction with the heat flow allows us to compute the entropic interpolation from $\mu_0$ to $\mu_1$, which is given by \cite{GigTamBB18,LeoSurvey,RipThesis}
\begin{equation}\label{eq:rhot}
\begin{aligned}
\mu^\epsilon_t&=\H^{\mu_0}_{t\ep}(\alpha^{\ep})\,\H^{\mu_1}_{(1-t)\ep}(\beta^{\ep}),\\
\H^{\mu}_s[f] &= \int_{\R^n} \frac{1}{\sqrt{2\pi s}}\exp\left(-\frac{1}{s}\| x-z\|^2\right)f(z)\mu(z){\rm d} z,
\end{aligned}
\end{equation}
and $\alpha^{\ep}$,$\beta^{\ep}$ are the Entropy-Kantorovich potentials solving the system \eqref{intro:SchSys}. In particular, we have that 
\begin{equation}
\label{eq:sch1}
\alpha^\ep(x) \H^{\mu_1}_{\ep}(\beta^\ep)(x)=1,\quad \beta^\ep(y)\,\H^{\mu_0}_{\ep}(\alpha^\ep)(y)=1.
\end{equation}
In particular, when we send the regularization parameter $\ep\to 0$, the curves of measures $\mu^\epsilon_t$ converge to the $2$-Wasserstein between $\mu_0$ and $\mu_1$ \cite{GigTam18,LeoSurvey}. Moreover, we can also write the  entropic interpolation $\mu^\epsilon_t$ and the dynamic entropic Kantorovich potentials $(\varphi^\ep_t,\psi^\ep_t)$ via the relation $\varphi^\ep_t + \psi^\ep_t = \ep\log \mu^\epsilon_t$.

Now, by defining $\phi_t^\ep = (\varphi^{\ep}_t-\psi^{\ep}_t)/2$, it is easy to check that by imposing $v^{\ep}_t = \nabla \phi^{\ep}_t$ we have that $(\mu^\epsilon_t,v^{\ep}_t)$ solves the Fokker-Planck equation
\begin{equation}\label{eq:fokker_plank}
\partial_t \mu^\epsilon_t + \nabla\cdot(v_t^{\ep}\mu^\epsilon_t) = \frac{\ep}{2}\Delta\mu^\epsilon_t.
\end{equation}

\section{Entropy-Regularized $2$-Wasserstein Distance between Gaussians}\label{sec:sink}

In this section we consider the special case of \eqref{eq:mainKL} and \eqref{def:sinkhorn} when $c(x,y) = d^2(x,y) = \vert x-y\vert^2$ is the Euclidian distance in $\R^n$ and $\mu_0 \sim \N(m_0,K_0)$,~$\nu \sim \N(m_1,K_1)$ are multivariate Gaussian distributions. We are interested in obtain explicity formulas for the optimal coupling $\gammaep$ solving \eqref{eq:mainKL}, the Entropy-Kantorovich maximizers $(\varphi^\epsilon,\psi^\epsilon)$ in \eqref{kanto} and the entropic displacement interpolation $\mu^{\ep}_t$ in \eqref{eq:rhot}.

We start by showing that we can assume, without loss of generality, that $\mu_0$ and $\mu_1$ are centered Gaussian distributions. The general case is obtain just by a shift depending on the $L^2$-distance of the center of both Gaussians. 

\begin{proposition}\label{prop:restriction_to_centered}
Let $c(x,y) = \|x-y\|^2$, $X_i\sim\mu_i\in \Probs(\R^n)$ for $i=0,1$ and $m_i = \E\left[\mu_i\right]$. Denote by $\hat{X}_i = X_i - m_i\sim \hat{\mu}_i$ the corresponding centered distributions. Then
\begin{equation}
\OT_{d^2}^\epsilon(\mu_0,\mu_1) = \|m_0-m_1\|^2
                                + \OT_{d^2}^\epsilon\left(\hat{\mu}_0,
                                \hat{\mu}_1\right).
\end{equation}
\end{proposition}
\begin{proof}
Recall the definition given in \eqref{eq:mainKL}
\begin{equation}
\OT_c^\epsilon(\mu_0, \mu_1) = \min_{\gamma\in \ADM(\mu_0,\mu_1)}\left\lbrace\E_\gamma[c]
+ \epsilon \KL(\gamma || \mu_0 \otimes \mu_1) \right\rbrace.
\end{equation}

Then, as $c=d^2$, for the first term we can write
\begin{equation}
    \begin{aligned}
        \E_\gamma\left[d^2\right] =& \int_{\R^n} \|x-y\|^2d\gamma(x,y)\\
        =& \int_{\R^n} \left(\|(x-m_0) - (y-m_1)\|^2 + \|m_0 - m_1\|^2 \phantom {\frac{1}{1}}\right.\\
        &\left.\phantom{\frac{1}{1}}+ 2\left((x-m_0) - (y-m_1)\right)^T(m_0-m_1)
        \right)d\gamma(x,y)\\
        &=  \|m_0-m_1\|^2 + \int_{\R^n}\|x-y\|^2d\gamma(x+m_0,y+m_1).
    \end{aligned}
\end{equation}
We now verify that the requirement $\gamma \in \ADM(\mu_0, \mu_1)$ is equivalent with $\gamma(\cdot + m_0, \cdot + m_1)\in \ADM(\hat{\mu}_0, \hat{\mu}_1)$, which results from 
\begin{equation}
    \int_{\R^n} \gamma(x+m_0, y+m_1) {\rm dy}= \mu_0(x+m_0)
    = \hat{\mu}_0(x),
\end{equation}
and similarly for the other margin. Finally, for the entropy term, we use the identity~\eqref{eq:kl_identity}. Now, as the entropy of a distribution does not depend on the expected value, we have $H(\mu_i) = H(\hat{\mu}_i)$, and therefore
\begin{equation}
    \KL(\gamma || \hat{\mu}_0 \otimes \hat{\mu}_1) = H(\mu_0) + H(\mu_1) - H(\gamma).
\end{equation}
Putting everything together, we get
\begin{equation}
\begin{aligned}
    \OT_{d^2}^\epsilon(\mu_0, \mu_1) =& \|m_0-m_1\|^2 \\
    &+ \min\limits_{\gamma \in \ADM(\hat{\mu}_0,\hat{\mu}_1)}\left\lbrace
    \E_\gamma[d^2] + \epsilon D_{\KL}(\gamma || \hat{\mu}_0 \otimes \hat{\mu}_1)
    \right\rbrace,\\
    =& \|m_0-m_1\|^2 + \OT_{d^2}^\epsilon(\hat{\mu}_0, \hat{\mu}_1).
\end{aligned}
\end{equation}
\qed
\end{proof}

\begin{proposition}\label{prop:centered_gaussian_plan}
Let $\mu_i = \Ncal(0,K_i)\in \Ncal(\R^n)$ for $i=0,1$. Then, the unique optimal plan $\gamma^\epsilon$ in $\OT_{d^2}^\epsilon(\mu_0, \mu_1)$ is a centered Gaussian distribution.
\end{proposition}
\begin{proof}
Note that $\E_\gamma[d^2]$ depends only on the mean and covariance of $\gamma$, and therefore remains constant, if $\gamma$ is replaced with a Gaussian with the corresponding mean and covariance (which we can do, as the marginals are Gaussians). Then, for the other term, using the identity~\eqref{eq:kl_identity}
\begin{equation}
    \KL(\gamma || \mu_0 \otimes \mu_1) = H(\mu) + H(\nu) - H(\gamma).
\end{equation}



It is readily seen that the $\gamma$ with a fixed covariance matrix minimizing this expression is Gaussian, as Gaussians achieve maximal entropy over distributions sharing a fixed covariance matrix. Therefore, we can deduce that $\gamma^\epsilon$ is Gaussian. Finally, as both of the marginals $\mu_0$ and $\mu_1$ are centered, so is $\gamma^\epsilon$.
\qed
\end{proof}

We now arrive at the main theorem of this work, detailing the entropic $2$-Wasserstein geometry between multivariate Gaussians. The proof is based on studying the Schr\"odinger system given in \eqref{intro:SchSys}. We give an alternative proof for the statement $\textbf{a.}$ in Theorem~\ref{thm:ent_was_gaussian} in Appendix~B, by finding the minimizer of the OT problem. Recall, that a noteworthy property of the entropic interpolant, is that even if we are interpolating from $\mu$ to itself, the trajectory does not constantly stay $\mu$.

\begin{theorem}\label{thm:ent_was_gaussian}
Let $\mu_i = \N(0,K_i)$, for $i=0,1$, be two centered multivariate Gaussian distributions in $\mathbb{R}^n$, write $N^\epsilon_{ij} = \left(I + \frac{16}{\epsilon^2}K_i^\frac{1}{2}K_jK_i^\frac{1}{2}\right)^\frac{1}{2}$ and
$M^\epsilon =  I + \left(I + \frac{16}{\epsilon^2}K_0K_1\right)^\frac{1}{2}$.
Then,
\begin{itemize}
    \item[\textbf{a.}] The density of the optimal entropy relaxed plan $\gamma^\epsilon$ is given by
    \begin{equation}\label{eq:ent_was_gauss_plan}
        \gamma^\epsilon(x,y) = \alpha^\epsilon(x)\beta^\epsilon(y) \exp\left(-\frac{\|x-y\|^2}{\epsilon}\right)\mu_0(x)\mu_1(y), 
    \end{equation}
    where $\alpha^\epsilon(x) = \exp\left(x^TAx + a\right)$, $\beta^\epsilon(y) = \exp\left(y^TBy + b\right)$, and
    \begin{equation}\label{eq:gaussian_system_final}
        \begin{aligned}
        A &= \frac{1}{4}K_0^{-\frac{1}{2}}\left(
        I + \frac{4}{\epsilon}K_0 - N^\epsilon_{01}
        \right)K_0^{-\frac{1}{2}}\\
        B &= \frac{1}{4}K_1^{-\frac{1}{2}}\left(
        I + \frac{4}{\epsilon}K_1 - N^\epsilon_{10}
        \right)K_1^{-\frac{1}{2}} \\
        \exp(a+b) &= \sqrt{
        \frac{1}{2^n} 
        \det\left(M^\epsilon\right)
        }.
        \end{aligned}
    \end{equation}
\item[\textbf{b.}] The entropic optimal transport quantity is given by
\begin{equation}\label{eq:gaus_ent_was_dist_cent}
\begin{aligned}
    \OT_{d^2}^\epsilon(\mu_0, \mu_1)
    =& \Tr(K_0) + \Tr(K_1)\\
    &- \frac{\epsilon}{2}\left(
    \Tr(M^\epsilon) - \log \det(M^\epsilon) + n\log 2 - 2n
    \right)
\end{aligned}
\end{equation}

\item[\textbf{c.}] The entropic displacement interpolation $\mu_t^\epsilon$, $t\in[0,1]$, between $\mu_0$ and $\mu_1$ is given by $\mu_t^\epsilon = \Ncal\left(0, K^{\epsilon}_t\right)$, where
\begin{equation}\label{eq:gaus_ent_was_geodesic}
\begin{aligned}
    K^\epsilon_t =& \frac{(1-t)^2\epsilon^2}{16}K_1^{-\frac{1}{2}}\left(
    -I + \left(
    \frac{4t}{(1-t)\epsilon}K_1 + N^\epsilon_{10}
    \right)^2
    \right)K_1^{-\frac{1}{2}}\\
    =& \frac{t^2\epsilon^2}{16}K_0^{-\frac{1}{2}}\left(
    -I + \left(
    \frac{4(1-t)}{t\epsilon}K_0 + N^\epsilon_{01}
    \right)^2
    \right)K_0^{-\frac{1}{2}}\\
     =& (1-t)^2K_0 + t^2K_1 + t(1-t)
    \left[\left(\frac{\epsilon^2}{16}I + K_0K_1\right)^{1/2}\right.\\
    &+\left. \left(\frac{\epsilon^2}{16}I + K_1K_0\right)^{1/2}\right].
\end{aligned}
\end{equation}
\end{itemize}
\end{theorem}

\begin{proof}
    \textbf{Part a.} Recall that $\alpha^\ep$, $\beta^\ep$ are the unique functions that give the density of the optimal plan $\gamma$
\begin{equation}
    \gamma^\epsilon(x,y) = \alpha^\ep(x)\beta^\ep(y)\exp\left(-\frac{\|x-y\|^2}{\epsilon}\right)\mu_0(x)\mu_1(y).
\end{equation}
The optimal plan is required to have the right marginals \eqref{intro:SchSys}, that is,
\begin{equation}\label{eq:gauss_ent_interpolant}
    \begin{aligned}
    \mu_0(x) &=
    \int_{\R^n} \gamma^\epsilon(x,y) {\rm d}y\\ &= \alpha^\ep(x)\int_{\R^n} \beta^\ep(y)\exp\left(-\frac{\|x-y\|^2}{\epsilon}\right)\mu_0(x)\mu_1(y){\rm d}y,\\
    \mu_1(y) &= \int_{\R^n} \gamma^\epsilon(x,y){\rm d}x\\ &= \beta^\ep(y)\int_{\R^n} \alpha^\ep(x)\exp\left(-\frac{\|x-y\|^2}{\epsilon}\right)\mu_0(x)\mu_1(y){\rm d}x.
    \end{aligned}
\end{equation}

Assuming $\alpha^\ep(x) = \exp(x^TAx+a)$ and $\beta^\ep(y) = \exp(y^TBy+b)$, substituting in $\mu_0$ and $\mu_1$, and after some simplifications, the system reads
\begin{equation}
    \begin{aligned}
     1 = &\frac{\exp(a+b)}{\sqrt{ \det(2\pi K_1)}}\exp\left(x^T\left(A - \frac{1}{\epsilon}I\right)x\right)\\
     & \times \int_X \exp\left(y^T\left(B-\frac{1}{\epsilon}I - \frac{1}{2}K_1^{-1}\right)y+\frac{2}{\epsilon}x^Ty\right){\rm d}y,\\
     1 = &\frac{\exp(a+b)}{\sqrt{ \det(2\pi K_0)}}\exp\left(y^T\left(B - \frac{1}{\epsilon}I\right)y\right)\\
     & \times\int_Y \exp\left(x^T\left(A-\frac{1}{\epsilon}I
     - \frac{1}{2}K_0^{-1}
     \right)x+\frac{2}{\epsilon}y^Tx\right){\rm d}x.
    \end{aligned}
    \label{eq:gaussian_system}
\end{equation}
Using the identity
\begin{equation}
    \int_X \exp\left(-x^TCx + b^Tx\right){\rm d}x 
    = \sqrt{\frac{\pi^n}{\det(C)}}
    \exp\left(\frac{1}{4}b^TC^{-1}b\right), 
\end{equation}
the system \eqref{eq:gaussian_system} results in
\begin{equation}\label{eq:gaussian_system2}
    \begin{aligned}
        A &= \frac{1}{\epsilon}I +\frac{1}{\epsilon^2}
        \left(B-\frac{1}{\epsilon}I  - \frac{1}{2}K_1^{-1} \right)^{-1},\\
        B &= \frac{1}{\epsilon}I  +\frac{1}{\epsilon^2}
        \left(A - \frac{1}{\epsilon}I  - \frac{1}{2}K_0^{-1} \right)^{-1},\\
        \exp(a+b) &=
        \sqrt{\det(2 K_1)
        \det\left(\frac{1}{\epsilon} I + \frac{1}{2}K_1^{-1}
        -B
        \right)
        }\\
        \exp(a+b) &=
        \sqrt{\det(2 K_0)
        \det\left(\frac{1}{\epsilon} I + \frac{1}{2}K_0^{-1}
        -A
        \right)
        }\\
    \end{aligned}
\end{equation}
Let us solve for $A$ and $B$ first. From system \eqref{eq:gaussian_system2}, we get that $A$ and $B$ can be written as
\begin{equation}
\begin{aligned}
    A &= \frac{1}{\epsilon}I 
    +\frac{1}{\epsilon^2}\left(
     \frac{1}{\epsilon^2}
    \left(A - \frac{1}{\epsilon}I - \frac{1}{2}K_0^{-1}\right)^{-1} - \frac{1}{2}K_1^{-1}
    \right)^{-1},\\
    B &= \frac{1}{\epsilon}I 
    +\frac{1}{\epsilon^2}\left(
     \frac{1}{\epsilon^2}
    \left(B - \frac{1}{\epsilon}I - \frac{1}{2}K_1^{-1}\right)^{-1} - \frac{1}{2}K_0^{-1}
    \right)^{-1}.
\end{aligned}
\label{eq:AB_system}
\end{equation}
Then, one can show, that the $A,B$ given in \eqref{eq:gaussian_system_final} solves this system.
Plugging $A,B$ in the expressions for $\exp(a+b)$ in \eqref{eq:gaussian_system2}, we get
\begin{equation}
    \exp(a+b) = \sqrt{
    \frac{1}{2^n} 
    \det\left(I + \left(I + \frac{16}{\epsilon^2}
    K_0K_1
    \right)^\frac{1}{2}\right)
    },
\end{equation}
for which a possible solution is given by
\begin{equation}
    a=b= \frac{1}{4}\left(
    -n\log 2 + \log \det\left(I + \left(
    I + \frac{16}{\epsilon^2}K_0K_1
    \right)^\frac{1}{2}\right)
    \right).
\end{equation}

Now, we show that $A$ solves the equation given in \eqref{eq:AB_system}. Manipulating \eqref{eq:AB_system} we see that it suffices to show the equality
\begin{equation}
    \left(A - \frac{1}{\epsilon}I \right)^{-1}
    = \left(A - \frac{1}{\epsilon}I - \frac{1}{2}K_0^{-1}\right)^{-1} - \frac{1}{2}K_1^{-1}.
\label{eq:showing_A}
\end{equation}
Substituting in $A$ given in \eqref{eq:gaussian_system_final}, the left-hand side reads
\begin{equation}
\left(A - \frac{1}{\epsilon}I\right)^{-1}
= 4K_0^\frac{1}{2}\left(
I - \left(I + \frac{16}{\epsilon^2}K_0^\frac{1}{2}K_1K_0^\frac{1}{2}\right)^\frac{1}{2}
\right)^{-1}K_0^\frac{1}{2},
\end{equation}
whereas the right-hand side is given by
\begin{equation}
\begin{aligned}
&\left(A - \frac{1}{\epsilon}I - \frac{1}{2}K_0^{-1}\right)^{-1} - \frac{1}{2}K_1^{-1}\\
 =& -4K_0^\frac{1}{2}\left(
 \frac{\epsilon^2}{8}\left(K_0^\frac{1}{2}K_1K_0^\frac{1}{2}\right)^{-1} +
 \left(I + \left(I + \frac{16}{\epsilon^2}K_0^\frac{1}{2}K_1K_0^\frac{1}{2}\right)^\frac{1}{2}\right)^{-1}
\right)K_0^\frac{1}{2}.
\end{aligned}
\end{equation}
Therefore, we need to show the equality
\begin{equation}
\begin{aligned}
  \left(I + \left(I + \frac{16}{\epsilon^2}K_0^\frac{1}{2}K_1K_0^\frac{1}{2}\right)^\frac{1}{2}
\right)^{-1}
=& \left(
 -I + \left(I + \frac{16}{\epsilon^2}K_0^\frac{1}{2}K_1K_0^\frac{1}{2}\right)^\frac{1}{2}
 \right)^{-1},\\
 &- \left(\frac{8}{\epsilon^2}K_0^\frac{1}{2}K_1K_0^\frac{1}{2}\right)^{-1}
 \end{aligned}
\end{equation}
which can be derived as follows
\begin{equation}
\begin{aligned}
   & - \left(\frac{8}{\epsilon^2}K_0^\frac{1}{2}K_1K_0^\frac{1}{2}\right)^{-1} + \left(
 -I + \left(I + \frac{16}{\epsilon^2}K_0^\frac{1}{2}K_1K_0^\frac{1}{2}\right)^\frac{1}{2}
 \right)^{-1}\\
 =& -2\left(I + \left(
 I + \frac{16}{\epsilon^2}K_0^\frac{1}{2}K_1K_0^\frac{1}{2}
 \right)^\frac{1}{2}\right)^{-1}
 + \left(-I + \left(
 I + \frac{16}{\epsilon^2}K_0^\frac{1}{2}K_1K_0^\frac{1}{2}
 \right)^\frac{1}{2}\right)^{-1} \\
 &\times \left(-I+\left(
 I + \frac{16}{\epsilon^2}K_0^\frac{1}{2}K_1K_0^\frac{1}{2}
 \right)^\frac{1}{2}\right)^{-1}\\
 =&\left(
 I - 2\left(I + \left(
 I + \frac{16}{\epsilon^2}K_0^\frac{1}{2}K_1K_0^\frac{1}{2}
 \right)^\frac{1}{2}\right)^{-1}
 \right) \left(-I + \left(
 I + \frac{16}{\epsilon^2}K_0^\frac{1}{2}K_1K_0^\frac{1}{2}
 \right)^\frac{1}{2}\right)^{-1}\\
 =&\left(
 I + \left(
 I + \frac{16}{\epsilon^2}K_0^\frac{1}{2}K_1K_0^\frac{1}{2}
 \right)^\frac{1}{2} - 2I
 \right)\left(I + \left(
 I + \frac{16}{\epsilon^2}K_0^\frac{1}{2}K_1K_0^\frac{1}{2}
 \right)^\frac{1}{2}\right)^{-1}
 \\
 &\times \left(-I + \left(
 I + \frac{16}{\epsilon^2}K_0^\frac{1}{2}K_1K_0^\frac{1}{2}
 \right)^\frac{1}{2}\right)^{-1}\\
 =&\left(I + \left(
 I + \frac{16}{\epsilon^2}K_0^\frac{1}{2}K_1K_0^\frac{1}{2}
 \right)^\frac{1}{2}\right)^{-1},
 \end{aligned}
\end{equation}
where the first step results from writing
\begin{equation}
- \left(\frac{8}{\epsilon^2}K_0^\frac{1}{2}K_1K_0^\frac{1}{2}\right)^{-1} = -2\left(-I + \left(I + \frac{16}{\epsilon^2}K_0^\frac{1}{2}K_1K_0^\frac{1}{2}\right)\right)^{-1},
\end{equation}
and using $M-I = (M^\frac{1}{2}+1)(M^\frac{1}{2}-1)$ on the right-hand side.

\textbf{Part b.} Let $\varphi_\epsilon(x) = \epsilon \log \alpha_\epsilon(x)$ and $\psi_\epsilon(y) = \epsilon \log \beta_\epsilon(y)$, and as previously,
\begin{equation}
    M^\epsilon = I + \left(I + \frac{16}{\epsilon^2}K_0K_1\right)^\frac{1}{2},
\end{equation}
then plugging $\varphi^\epsilon$ and $\psi^\epsilon$ into \eqref{kanto} yields
\begin{equation}
\begin{aligned}
    \OT_{d^2}^\epsilon(\mu_0, \mu_1)  =& \E_{\mu_0}[\varphi_\epsilon] + \E_{\mu_1}[\psi_\epsilon]\\
    &-\epsilon\left(\E_{\mu_0\otimes \mu_1}\left[\exp\left(\frac{1}{\epsilon}
    \left((\varphi \oplus \psi) - d^2\right)
    \right)\right]-1\right)
    \\
    =& \epsilon\left(\E_{\mu_0}[\log \alpha_\epsilon] + \E_{\mu_1}[\log \beta_\epsilon]
    \right)\\
    &- \epsilon \left(\E_{\mu_0 \otimes \mu_1}\left[
    \alpha^\epsilon \beta^\epsilon \exp\left(-\frac{1}{\epsilon}d^2\right)
    \right]-1\right) \\
    =& \epsilon\left(
    \E_{X\sim\mu_0}\left[X^TAX + a\right] + 
    \E_{Y\sim\mu_1}\left[Y^TBY + b\right]
    \right) \\
    =& \epsilon\left(
    \Tr\left[K_0A\right] + \Tr\left[K_1B\right] + a + b
    \right)\\
    =& \frac{\epsilon}{4}
    \Tr\left[
    I + \frac{4}{\epsilon} K_0 - \left(
    I + \frac{16}{\epsilon^2} K_0^\frac{1}{2}K_1K_0^\frac{1}{2}
    \right)^\frac{1}{2}
    \right] \\
    &+ \frac{\epsilon}{4}\Tr\left[
    I + \frac{4}{\epsilon} K_1 - \left(
    I + \frac{16}{\epsilon^2} K_1^\frac{1}{2}K_0K_1^\frac{1}{2}
    \right)^\frac{1}{2}
    \right]\\
    &+ \epsilon(a + b)\\
    =& \Tr K_0 + \Tr K_1 - \frac{\epsilon}{2}\left(
    \Tr M^\epsilon - \log\det M^\epsilon + n\log 2 - 2n
    \right),
\end{aligned}
\end{equation}
where we used the fact that $C^\frac{1}{2}DC^\frac{1}{2}$ has same eigenvalues and thus trace as $CD$ for any square matrices $C$ and $D$.

\textbf{Part c.} As we have solved for $\alpha^\epsilon$ and $\beta^\epsilon$ for the optimal plan, the entropic interpolant $\mu_t^\epsilon$ between $\mu_0$ and $\mu_1$ is given by \eqref{eq:rhot}, which we rewrite here
\begin{equation}
    \mu^\epsilon_t(x) = \left(\H^{\mu_0}_{t\epsilon}[\alpha^\epsilon](x)\right)\left(\H^{\mu_1}_{(1-t)\epsilon}[\beta^\epsilon](x)\right).
\end{equation}
Then, we can compute
\begin{equation}
\begin{aligned}
    \H^{\mu_0}_{t\epsilon}[\alpha^\epsilon](x) &=
    \frac{1}{\sqrt{\det((2\pi)^2t\epsilon K_0)}}
    \int_{\R^n} \exp\left(z^TAz + a-\frac{1}{t\epsilon}\|x-z\|^2-\frac{1}{2}z^TK_0^{-1}z\right){\rm d}z\\
    =&
    \frac{\exp\left(a-\frac{1}{t\epsilon}x^Tx\right)}{\sqrt{\det(2\pi t\epsilon K_0)}}
    \int_{\R^n} \exp\left(z^T\left(A - \frac{1}{t\epsilon}I - \frac{1}{2}K_0^{-1}\right)z + \frac{2}{t\epsilon}x^Tz\right){\rm dz} \\
    =& \frac{\exp(a)}{\sqrt{\det\left(2\pi t\epsilon K_0\right)
    \det\left(\frac{1}{t\epsilon}I + \frac{1}{2}K_0^{-1}-A\right)}}\\
    &\times
    \exp\left(
    \frac{1}{t^2\epsilon^2}x^T\left(
    \left(\frac{1}{t\epsilon}I + \frac{1}{2}K_0^{-1}-A\right)^{-1}
    - I
    \right)x
    \right),
\end{aligned}
\end{equation}
similar computation yields
\begin{equation}
\begin{aligned}
    &\H^{\mu_1}_{(1-t)\epsilon}[\beta^\epsilon](x)\\
    =& \frac{\exp(b)}{\sqrt{\det\left(2\pi (1-t) \epsilon K_1\right)
    \det\left(\frac{1}{(1-t)\epsilon}I + \frac{1}{2}K_1^{-1}-B\right)}}\\
    &\times \exp\left(
    \frac{1}{(1-t)^2\epsilon^2}x^T\left(
    \left(\frac{1}{(1-t)\epsilon}I + \frac{1}{2}K_1^{-1}-B\right)^{-1}
    - I
    \right)x
    \right).
\end{aligned}
\end{equation}
Putting these together, we get
\begin{equation}
\begin{aligned}
    \mu^\epsilon_t(x)&= \left(\H^{\mu_0}_{t\epsilon}[\alpha^\epsilon](x)\right)\left(\H^{\mu_1}_{(1-t)\epsilon}[\beta^\epsilon](x)\right)\\
    =& N\exp\left(
    x^T\left[
    \frac{1}{t^2\epsilon^2}\left(
    \left(\frac{1}{t\epsilon}I + \frac{1}{2}K_0^{-1}- A\right)^{-1}-I
    \right)\right.\right.\\
    &\left.\left.+\frac{1}{(1-t)^2\epsilon^2}\left(
    \left(\frac{1}{(1-t)\epsilon}I + \frac{1}{2}K_1^{-1}- B\right)^{-1}-I
    \right)
    \right]x
    \right)\\
    :=& N\exp\left(x^T\left(T_0(A) + T_1(B)\right)x\right),
\end{aligned}
\label{eq:interpolant_density}
\end{equation}
where $N$ is a normalizing constant. We can simplify the matrix $T_0(A) + T_1(B)$ in \eqref{eq:interpolant_density}. Write 
\begin{equation}
N^\epsilon_{10} = \left(I + \frac{16}{\epsilon^2}K_1^\frac{1}{2}K_0K_1^\frac{1}{2}\right)^\frac{1}{2},
\end{equation}
and consider the first term
\begin{equation}
\begin{aligned}
    T_0(A) & =\frac{1}{t^2\epsilon^2}\left(
    \left(\frac{1}{t\epsilon}I + \frac{1}{2}K_0^{-1}- A\right)^{-1}-I
    \right)\\
    &= \frac{1}{t^2\epsilon^2}\left(\left(
    \frac{(1-t)}{t\epsilon}I + \frac{1}{2}K_0^{-1}
    - \frac{1}{\epsilon^2}\left( B - \frac{1}{\epsilon}I - \frac{1}{2} K_1^{-1}\right)^{-1}
    \right)^{-1} - t\epsilon I\right)\\
    &= \frac{1}{t^2\epsilon}\left(\left(
    \frac{(1-t)}{t}I
    -  \left( I - \epsilon B\right)^{-1}
    \right)^{-1} - t\epsilon I\right)\\
    &= \frac{1}{t^2\epsilon}\left(\left(
    \frac{t}{(1-t)}I - \frac{t^2}{(1-t)^2}\left(
    \frac{t}{(1-t)}I + (I-\epsilon B)
    \right)^{-1}
    \right) - t\epsilon I\right)\\
    &= \frac{1}{(1-t)^2\epsilon^2}\left(
    I - \left(\frac{1}{(1-t)\epsilon}I -B\right)^{-1}\right)\\
    &= \frac{4}{(1-t)^2\epsilon^2}K_1^\frac{1}{2}\left(
    -I + \frac{4t}{\epsilon(1-t)}K_1 + N^\epsilon_{10}
    \right)^{-1}K_1^\frac{1}{2},
\end{aligned}
\end{equation}
where second equality follows from \eqref{eq:gaussian_system2}, third from \eqref{eq:AB_system}, fourth from the Woodbury matrix inverse identity 
\begin{equation}\label{eq:woodbury}
    \left(C+D\right)^{-1} = C^{-1} - C^{-1}\left(C^{-1} + D^{-1}\right)^{-1}C^{-1},
\end{equation}
and the last one from substituting in $B$ given in \eqref{eq:gaussian_system_final}.

Likewise, we can substitute $B$ in the second term $T_1(B)$, which yields
\begin{equation}
\begin{aligned}
T_1(B) &=\frac{1}{(1-t)^2\epsilon^2}\left(
    \left(\frac{1}{(1-t)\epsilon}I + \frac{1}{2}K_1^{-1}- B\right)^{-1}-I
    \right) \\
    &= 
    \frac{4}{(1-t)^2\epsilon^2}K_1^\frac{1}{2}\left(
    I + \frac{4t}{\epsilon(1-t)}K_1 + N^\epsilon_{10}
    \right)^{-1}K_1^\frac{1}{2}.
\end{aligned}
\end{equation}

Putting the two terms together, we get
\begin{equation}
\begin{aligned}
T_0(A) + T_1(B)  =&     \frac{4}{(1-t)^2\epsilon^2}K_1^\frac{1}{2}\left(\left(
    -I + \frac{4t}{\epsilon(1-t)}K_1 + N^\epsilon_{10}
    \right)^{-1}\right.\\
    &+
    \left.\left(I + \frac{4t}{\epsilon(1-t)}K_1 + N^\epsilon_{10}
    \right)^{-1}
    \right)K_1^\frac{1}{2} \\
=& \frac{8}{(1-t)^2\epsilon^2}K_1^\frac{1}{2}\left(
    I - \left(
    \frac{4t}{(1-t)\epsilon}K_1 + N^\epsilon_{10}
    \right)^2
    \right)^{-1}K_1^\frac{1}{2}. \\
\end{aligned}
\end{equation}
Note, that we can write \eqref{eq:interpolant_density} as a Gaussian with covariance matrix $K_t$
\begin{equation}
\begin{aligned}
    \mu^\epsilon_t(x) &= N \exp\left(x^T\left(T_0(A)+T_1(B)\right)x\right)\\
    &=N \exp\left(
    -\frac{1}{2}x^T\left(K^\epsilon_t\right)^{-1}x,
    \right)
\end{aligned}
\end{equation}
and so
\begin{equation}
\begin{aligned}
    K^\epsilon_t &= -\frac{1}{2}\left(T_0(A) + T_1(B)\right)^{-1}\\
    &= \frac{(1-t)^2\epsilon^2}{16}K_1^{-\frac{1}{2}}\left(
    -I + \left(
    \frac{4t}{(1-t)\epsilon}K_1 + N^\epsilon_{10}
    \right)^2
    \right)K_1^{-\frac{1}{2}}\\
    &= (1-t)^2K_0 + t^2K_1 + t(1-t)
\left[\left(\frac{\epsilon^2}{16}I + K_0K_1\right)^{1/2} + \left(\frac{\epsilon^2}{16}I + K_1K_0\right)^{1/2}\right].
\end{aligned}
\end{equation}
Where for the last step we use the formula
\begin{equation}
\left(I + \frac{16}{\epsilon^2}K_0K_1\right)^{1/2} = K_0^{1/2}\left(I + \frac{16}{\epsilon^2}K_0^{1/2}K_1K_0^{1/2}\right)^{1/2} K_0^{-1/2}.
\end{equation}
\qed
\end{proof}

 Above we only considered centered Gaussians. Now we combine the results obtained in Proposition~\ref{prop:restriction_to_centered} and Theorem~\ref{thm:ent_was_gaussian} to deduce the general case. As a consequence, we also derive the corresponding formulas for the Sinkhorn divergence between two Gaussians

\begin{corollary}
Let $\mu_i = \N(m_i,K_i)$, for $i=0,1$, be two multivariate Gaussian distributions in $\mathbb{R}^n$. Then,
\begin{itemize}
\item[\textbf{a.}]
\begin{equation}\label{eq:gaus_ent_was_dist}
\begin{aligned}
    \OT_{d^2}^\epsilon(\mu_0, \mu_1)
    =&\|m_0-m_1\|^2
    + \Tr(K_0) + \Tr(K_1)\\
    &- \frac{\epsilon}{2}\left(
    \Tr(M^\epsilon) - \log \det(M^\epsilon) + n\log 2 - 2n
    \right)
\end{aligned}
\end{equation}

\item[\textbf{b.}] The entropic interpolant between $\mu_0$ and $\mu_1$ is $\mu_t^\epsilon = \Ncal\left(m_t, K_t\right)$, $t\in[0,1]$, where $m_t = (t-1)m_0 - tm_1$, and $K_t$ is given in~\eqref{eq:gaus_ent_was_geodesic}.

\item[\textbf{c.}] Write $M_{ij}^\epsilon = I + \left(I + \frac{16}{\epsilon^2}K_iK_j\right)^\frac{1}{2}$, then
\begin{equation}\label{eq:gauss_sinkhorn}
\begin{aligned}
    S_2^\epsilon(\mu_0,\mu_1)
    =& \|m_0 - m_1\|_2^2 + \frac{\epsilon}{4} \left(
    \Tr\left(
    M_{00}^\epsilon - 2 M_{01}^\epsilon + M_{11}^\epsilon
    \right)\phantom{\frac{M^2}{M^2}}\right.\\
    &+ \left.\log \left(
    \frac{\det^2( M_{01}^\epsilon )}{\det(M_{00}^\epsilon)\det(M_{11}^\epsilon)}
    \right)\right).
\end{aligned}
\end{equation}
\end{itemize}
\end{corollary}

We will now emphasize an identity that can be derived from the calculations of Theorem~\ref{thm:ent_was_gaussian}, which we find useful.
\begin{lemma}\label{lemma:schrodinger_system_identity}
    Let $C,D$ be symmetric positive-definite matrices. Then,
\begin{equation}
\begin{aligned}
    &\frac{4}{\epsilon} D^\frac{1}{2}\left(
    I + \left(I + \frac{16}{\epsilon^2}D^\frac{1}{2}C D^\frac{1}{2}\right)^\frac{1}{2}
    \right)^{-1}D^\frac{1}{2}\\
    =& I - \frac{\epsilon}{4}C^{-\frac{1}{2}}\left(
    I + \frac{4}{\epsilon}C - \left(I + \frac{16}{\epsilon^2}C^\frac{1}{2}DC^\frac{1}{2}\right)^\frac{1}{2}
    \right)C^{-\frac{1}{2}}.
\end{aligned}
\end{equation}
\end{lemma}
\begin{proof}
Similarly to \eqref{eq:gaussian_system_final}, let 
\begin{equation}
    \begin{aligned}
        A &= \frac{1}{4}C^{-\frac{1}{2}}\left(
        I + \frac{4}{\epsilon}C - \left(I + \frac{16}{\epsilon^2}C^\frac{1}{2}DC^\frac{1}{2}\right)^\frac{1}{2}
        \right)C^{-\frac{1}{2}}\\
        B &= \frac{1}{4}D^{-\frac{1}{2}}\left(
        I + \frac{4}{\epsilon}D - \left(I + \frac{16}{\epsilon^2}D^\frac{1}{2}CD^\frac{1}{2}\right)^\frac{1}{2}
        \right)D^{-\frac{1}{2}}.
    \end{aligned}
\end{equation}
Then, substituting $B$ into the first equation of \eqref{eq:gaussian_system2} (while remembering to replace $K_0 \mapsfrom C$, $K_1 \mapsfrom D$) results in
\begin{equation}
\begin{aligned}
A =& \frac{1}{\epsilon}I +\frac{1}{\epsilon^2}
        \left(B-\frac{1}{\epsilon}I  - \frac{1}{2}D^{-1} \right)^{-1}\\
    =& \frac{1}{\epsilon}I + \frac{1}{\epsilon^2}\left(
    \frac{1}{4}D^{-\frac{1}{2}}\left(
        I + \frac{4}{\epsilon}D - \left(I + \frac{16}{\epsilon^2}D^\frac{1}{2}CD^\frac{1}{2}\right)^\frac{1}{2}
        \right)D^{-\frac{1}{2}}\right.\\
        &\left.\phantom{\frac{K}{K}^\frac{1}{2}}-\frac{1}{\epsilon}I - \frac{1}{2}D^{-1}
    \right)^{-1}\\
    =& \frac{1}{\epsilon}I - \frac{4}{\epsilon^2}D^\frac{1}{2}\left(I +
    \left(I + \frac{16}{\epsilon^2}D^\frac{1}{2}CD^\frac{1}{2}\right)^\frac{1}{2}    
    \right)^{-1}D^\frac{1}{2},
\end{aligned}
\end{equation}
and so the result follows from substituting in $A$, multiplying both sides by $-\epsilon$, and moving $-I$ from right-hand side to left-hand side.
\qed
\end{proof}

Next, we  study the limiting cases of $\epsilon$ going to $0$ and $\infty$, reconfirming that the Sinkhorn divergence interpolates between $2$-Wasserstein and $MMD$~\cite{feydy18,genevay17,ramdas2017}.
\begin{proposition}\label{prop:epsilon_convergence}
Let $\mu_i = \N(m_i,K_i)$, for $i=0,1$, be two multivariate Gaussian distributions in $\mathbb{R}^n$. Then,
\begin{itemize}
    \item[\textbf{a.}]
    \begin{equation}
    \begin{aligned}
        \lim\limits_{\epsilon \to 0} \OT^\epsilon_{d^2}(\mu_0,\mu_1) &= W_2^2(\mu_0, \mu_1)\\
        \lim\limits_{\epsilon \to \infty} \OT^\epsilon_{d^2}(\mu_0, \mu_1) &=
        \|m_0-m_1\|^2 + \Tr(K_0) + \Tr(K_1)
    \end{aligned}
    \end{equation}
    \item[\textbf{b.}]
    \begin{equation}
    \begin{aligned}
        \lim\limits_{\epsilon \to 0} S^\epsilon_{2}(\mu_0,\mu_1) &= W_2^2(\mu_0, \mu_1) \\
        \lim\limits_{\epsilon \to \infty} S^\epsilon_{2}(\mu_0,\mu_1) &= \|m_0-m_1\|^2 \\
    \end{aligned}
    \end{equation}
    \item[\textbf{c.}] For $t\in [0,1]$, denote by $\mu_t$ the $2$-Wasserstein geodesic given in \eqref{eq:gaus_was_geodesic}, and by $\mu_t^\epsilon$ the entropic $2$-Wasserstein interpolant between $\mu_0$ and $\mu_1$ given in \eqref{eq:gaus_ent_was_geodesic}. Then,
    \begin{equation}
        \lim\limits_{\epsilon \to 0} \mu_t^\epsilon = \mu_t.
    \end{equation}
\end{itemize}
\end{proposition}
\begin{proof}
    \textbf{Part a.} The $\epsilon\to 0$ case is a straight-forward computation
   \begin{equation}
    \begin{aligned}
    \OT_{d^2}^\epsilon(\mu_0, \mu_1)
    =& \|m_0 - m_1\|^2
    + \Tr(K_0) + \Tr(K_1)\\
    &- \frac{\epsilon}{2}\left(
    \Tr(M^\epsilon) - \log \det(M^\epsilon) + n\log 2 - 2n
    \right)\\
    =& \|m_0-m_1\|^2 + \Tr(K_0) + \Tr(K_1)\\
    &- 2 \Tr\left(\frac{\epsilon}{4}I + \left(\frac{\epsilon^2}{16}I + K_0K_1\right)^\frac{1}{2}\right)\\
    &+ \frac{\epsilon}{2}\log\left(\det\left(
    \frac{\epsilon}{4}I + \left(\frac{\epsilon^2}{16}I + K_0K_1\right)^\frac{1}{2}\right)\right)\\
    &+\frac{\epsilon n}{2}(\log{2}-\log{\epsilon}+2).
    \end{aligned}
    \end{equation}
    
Therefore, since $\epsilon \log \epsilon \to 0$ when $\epsilon \to 0$,
    \begin{equation}
    \begin{aligned}
    \lim_{\ep\to 0}\OT_{d^2}^\epsilon(\mu_0, \mu_1) &=  \|m_0-m_1\|^2 + \Tr(K_0) + \Tr(K_1) - 2 \Tr\left(K_0K_1\right)^\frac{1}{2}\\
    &= W_2^2(\mu_0,\mu_1).
    \end{aligned}
    \end{equation}

We now compute the limit when $\ep\to\infty$. It is enough to show that the term 
\begin{equation}
\frac{\epsilon}{2}\left(\Tr(M^\epsilon) - \log\det\left(M^\epsilon\right) + n\log 2 - 2n\right),
\end{equation}
goes to 0 when $\ep\to\infty$. In fact, denote by $\{\lambda_i\}_{i=1}^n$ the eigenvalues of $K_1K_2$. Then, 
    \begin{equation}
    \begin{aligned}
        &\frac{\epsilon}{2}\left(\Tr(M^\epsilon) - \log\det\left(M^\epsilon\right) + n\log 2 - 2n\right)\\
        =&\frac{\epsilon}{2}\sum_{i=1}^n \left(-1 + \left(1 + \frac{16}{\epsilon^2}\lambda_i\right)^\frac{1}{2} - \log\left(\frac{1}{2}\left(
        1 + \left(1+\frac{16}{\epsilon^2}\lambda_i\right)^\frac{1}{2}
        \right)\right)\right).
    \end{aligned}
    \end{equation} 
    
So, first notice that for any $\lambda > 0$, 
    \begin{equation}
        \epsilon\left(-1 + \left(1 + \frac{16}{\epsilon^2}\lambda\right)^\frac{1}{2}\right)= \frac{16\lambda}{\epsilon + \left(\epsilon + 16\lambda\right)^\frac{1}{2}}
        \overset{\epsilon\to \infty}{=} 0.
    \end{equation}
    Second, we have 
    \begin{equation}
        \begin{aligned}
        &\lim\limits_{\epsilon \to \infty}\epsilon \log\left(\frac{1}{2}\left(
        1 + \left(1 + \frac{16}{\epsilon^2}\lambda \right)^\frac{1}{2}
        \right)\right)\\
        \overset{\mathrm{L'Hospital}}{=}& \lim\limits_{\epsilon\to\infty}
        \frac{16\lambda}{\epsilon^3\left(1+\left(1 + \frac{16}{\epsilon^2}\lambda \right)^\frac{1}{2}\right)\left(1 + \frac{16}{\epsilon^2}\lambda \right)^\frac{1}{2}\log^2\left(
        \frac{1}{2}\left(1 + \left(1 + \frac{16}{\epsilon^2}\lambda \right)^\frac{1}{2}\right)
        \right)}\\
        =&0,
        \end{aligned}
    \end{equation}
     and so the result follows.
    
    \textbf{Part b.} Straight-forward application of the above result to~\eqref{eq:gauss_sinkhorn}.
    
    \textbf{Part c.}
    By a straight-forward computation on \eqref{eq:gaus_ent_was_geodesic},
    \begin{equation}
    \begin{aligned}
        K^{\epsilon}_t=& (1-t)^2K_0 + t^2K_1 + t(1-t)
    \left[\left(\frac{\epsilon^2}{16}I + K_0K_1\right)^{1/2}\right.\\
    &+\left. \left(\frac{\epsilon^2}{16}I + K_1K_0\right)^{1/2}\right]
    \\
    \overset{\epsilon\to 0}{=}&(1-t)^2K_0 + t^2K_1 + t(1-t) [(K_0K_1)^{1/2} + (K_1K_0)^{1/2}]\\
    =& K_t.
    \end{aligned}
    \end{equation}
\qed
\end{proof}

\section{Entropic and Sinkhorn Barycenters}\label{sec:barycenter}
In this section, we compute barycenters under the entropic regularization of the $2$-Wasserstein distance (e.g. \cite{benamou15,bigot19,cazelles17,doucet14,DMaGer19,kroshnin19,lin19}) and the $2$-Sinkhorn divergence of a population of multivariate Gaussians, restricted to the manifold of Gaussians.

\textbf{Entropic $2$-Wasserstein barycenter.}  Given $N$ probability measures $\mu_i\in \Probs(\R^n)$, $i=1,2,..,N$, the entropic barycenter $\bar{\mu}$ with weights $\lambda_i\geq 0$ is defined in the vein  of  \emph{Karcher} and \emph{Fr\'echet means}, given as

\begin{equation}\label{def:barycenter}
    \bar{\mu} := \argmin\limits_{\mu \in \Probs(\R^n)} \sum_{i=1}^N \lambda_i\OT^\epsilon_{d^2}(\mu, \mu_i), \quad \sum^N_{i=1}\lambda_i = 1.
\end{equation}
Then, \eqref{def:barycenter} is strictly convex, as $\OT_c^\epsilon(\mu,\nu)$ is strictly convex in both $\mu$ and $\nu$ as stated by Prop.~\ref{prop:convex_entropic_transport}.

Next, let us focus on the Gaussian case. We lack the proof that such a barycenter will indeed be a Gaussian, so do note, that the following statement requires the restriction to Gaussians for the candidate barycenters.

\begin{theorem}[Entropic Barycenter of Gaussians]\label{thm:entropic_was_gaus_barycenter}
Let $\mu_i=\Ncal\left(m_i,K_i\right)$, $i=1,2,...,N$ be a population of multivariate Gaussians. Then, their entropic barycenter \eqref{def:barycenter} with weights $\lambda_i\geq 0$ such that $\sum^N_{i=1}\lambda_i = 1$, restricted to the manifold of Gaussians $\Ncal(\R^n)$, is given by $\bar{\mu}=\Ncal(\bar{m}, \bar{K})$, where
\begin{equation}\label{eq:optimal_k_entropic_barycenter}
\begin{aligned}
    \bar{m} = \sum_{i=1}^N \lambda_i m_i, \quad \bar{K}  = \frac{\epsilon}{4}\sum_{i=1}^N\lambda_i\left(-I + \left(I + \frac{16}{\epsilon^2}\bar{K}^\frac{1}{2}K_i\bar{K}^\frac{1}{2}\right)^\frac{1}{2}\right).
\end{aligned}
\end{equation}
\end{theorem}
\begin{proof}
Proposition~\ref{prop:restriction_to_centered} allows us to split the geometry into the $L^2$-geometry between the means and the entropic $2$-Wasserstein geometry between the centered Gaussians (or their covariances). Then, it immediately follows that
\begin{equation}
    \bar{m} = \sum_{i=1}^N \lambda_i m_i.
\end{equation}
Therefore, we restrict our analysis to the case of centered distributions. Remark again, that in general, the minimizer of \eqref{def:barycenter} might not be Gaussian, even when the population consists of Gaussians. However, here we will look for the barycenter on the manifold of Gaussian measures.

We begin with a straight-forward computation of the gradient of the objective given in~\eqref{def:barycenter}
\begin{equation}\label{eq:barycenter_proof0}
    \begin{aligned}
        &\nabla_K\sum_{i=1}^N \lambda_i \OT^\epsilon_{d^2}\left(\Ncal(0,K),\Ncal(0,K_i)\right)\\
        =& \nabla_K \sum_{i=1}^N \lambda_i \left(
        \Tr K + \Tr K_i - \frac{\epsilon}{2}
        \Tr\left(I + \left(I + \frac{16}{\epsilon^2}K_i^\frac{1}{2}KK_i^\frac{1}{2}\right)^\frac{1}{2}
        \right)\right.\\
        &+ \frac{\epsilon}{2}\log\det\left(
        I + \left(I + \frac{16}{\epsilon^2}K_i^\frac{1}{2}KK_i^\frac{1}{2}\right)^\frac{1}{2}
        \right)\\
        &\left. \phantom{\frac{\sum^2}{\sum^2}}-\frac{\epsilon}{2}\left(n\log 2 - 2n\right)
        \right),\\
        =& \sum_{i=1}^N \lambda_i 
        \left(\nabla_K \Tr K  - \frac{\epsilon}{2}
        \nabla_K\Tr\left(I + \left(I + \frac{16}{\epsilon^2}K_i^\frac{1}{2}KK_i^\frac{1}{2}\right)^\frac{1}{2}
        \right)\right.\\
        &+\left.\frac{\epsilon}{2}\nabla_K\log\det\left(
        I + \left(I + \frac{16}{\epsilon^2}K_i^\frac{1}{2}KK_i^\frac{1}{2}\right)^\frac{1}{2}
        \right)\right).
    \end{aligned}
\end{equation}
where we used the closed-form solution obtained in the part \textbf{b.} of Theorem \ref{thm:ent_was_gaussian}.  Now, recall that $\nabla_K \Tr K = I$. For the second term, it holds 
\begin{equation}\label{eq:barycenter_proof1}
\begin{aligned}
    &\nabla_K \Tr\left(I + \left(I + \frac{16}{\epsilon^2}K_i^\frac{1}{2}KK_i^\frac{1}{2}\right)^\frac{1}{2}\right)\\
    =& \frac{8}{\epsilon^2}K_i^\frac{1}{2}\Tr\left(I + \left(I + \frac{16}{\epsilon^2}K_i^\frac{1}{2}KK_i^\frac{1}{2}\right)^\frac{1}{2}\right)^{-1}K_i^\frac{1}{2}.
\end{aligned}
\end{equation}
Finally, for the third term, we have  
\begin{equation}\label{eq:barycenter_proof2}
    \begin{aligned}
    &\nabla_K \log\det\left(
    I + \left(I + \frac{16}{\epsilon^2}K_i^\frac{1}{2}KK_i^\frac{1}{2}\right)^\frac{1}{2}
    \right)\\
    &= \nabla_K \trace\left(
    \Log\left(
    I + \left(I + \frac{16}{\epsilon^2}K_i^\frac{1}{2}KK_i^\frac{1}{2}\right)^\frac{1}{2}
    \right)
    \right)\\
    &=
    \frac{8}{\epsilon^2}K_i^\frac{1}{2}\left(
    \left(I + \frac{16}{\epsilon^2}K_i^\frac{1}{2}KK_i^\frac{1}{2}\right) + 
    \left(I + \frac{16}{\epsilon^2}K_i^\frac{1}{2}KK_i^\frac{1}{2}\right)^\frac{1}{2}
    \right)^{-1}K_i^\frac{1}{2},
    \end{aligned}
\end{equation}   
where $\Log(M)$ denotes the matrix square-root, and we use the results
\begin{equation}
    \log\det(M) = \Tr\left(\Log(M)\right),\quad \nabla_M \Tr f(M) = f'(M),
\end{equation}
when $f$ is a matrix function given by a Taylor series, such as the matrix square-root or the matrix logarithm. 

Substituting \eqref{eq:barycenter_proof1} and \eqref{eq:barycenter_proof2} in \eqref{eq:barycenter_proof0}, and using the Woodbury matrix inverse identity ~\eqref{eq:woodbury}, we get
\begin{equation}\label{eq:barycenter_proof3}
\begin{aligned}
    &\nabla_K\sum_{i=1}^N \lambda_i \OT^\epsilon_{d^2}\left(\Ncal(0,K),\Ncal(0,K_i)\right)\\
    =&\sum_{i=1}^N \lambda_i\left(
    I - \frac{4}{\epsilon} K_i^\frac{1}{2}\left(
    I + \left(I + \frac{16}{\epsilon^2}K_i^\frac{1}{2}KK_i^\frac{1}{2}\right)^\frac{1}{2}
    \right)^{-1}K_i^\frac{1}{2}\right)\\
    =& \frac{\epsilon}{4}\sum^N_{i=1}\lambda_iK^{-\frac{1}{2}}\left(
    I + \frac{4}{\epsilon}K - \left(I + \frac{16}{\epsilon^2}K^\frac{1}{2}K_iK^\frac{1}{2}\right)^\frac{1}{2}
    \right)K^{-\frac{1}{2}}.
\end{aligned}
\end{equation}
The last equality follows from Lemma~\ref{lemma:schrodinger_system_identity} with the substitutions $C\mapsfrom K$ and $D\mapsfrom K_i$. Finally, setting \eqref{eq:barycenter_proof3} to zero, we get that the optimal $\bar{K}$ satisfies the expression given in~\eqref{eq:optimal_k_entropic_barycenter}.
\qed
\end{proof}

\textbf{Sinkhorn barycenter.} Now, we compute the barycenter of a population of Gaussians under the Sinkhorn divergence, defined by
\begin{equation}\label{def:sinkhorncenter}
    \bar{\mu} := \argmin\limits_{\mu \in \Probs(\R^n)} \sum_{i=1}^N \lambda_i S^\epsilon_2(\mu, \mu_i), \quad \lambda_i\geq 0 \text{ and } \sum^N_{i=1}\lambda_i = 1.
\end{equation}
Note that as $S_\epsilon^2(\mu,\nu)$ is convex in both $\mu$ and $\nu$~\cite[Thm. 1]{feydy18}, and so \eqref{def:sinkhorncenter} is convex in $\mu$. Now, similarly to the entropic barycenter case, we look for the barycenter of a population of Gaussians in the space of Gaussians $\Ncal(\R^n)$.

\begin{theorem}[Sinkhorn Barycenter of Gaussians]\label{thm:sinkhornbarycenter}
Let $\mu_i=\Ncal\left(m_i,K_i\right)$, $i=1,2,...,N$ be a population of multivariate Gaussians. Then, their Sinkhorn barycenter \eqref{def:sinkhorncenter} with weights $\lambda_i\geq 0$ such that $\sum^N_{i=1}\lambda_i = 1$, restricted to the manifold of Gaussians $\Ncal(\R^n)$, is given by $\bar{\mu}=\Ncal(\bar{m}, \bar{K})$, where
\begin{equation}\label{eq:sinkhorn_barycenter}
    \bar{m} = \sum_{i=1}^N \lambda_i m_i, \quad 
    \bar{K}  = \frac{\epsilon}{4}\left(-I + \left(\sum_{i=1}^N\lambda_i\left(I + \frac{16}{\epsilon^2}\bar{K}^\frac{1}{2}K_i\bar{K}^\frac{1}{2}\right)^\frac{1}{2}\right)^2
    \right)^\frac{1}{2}.
\end{equation}
\end{theorem}
\begin{proof}
As in the entropic $2$-Wasserstein case, we take $\mu=\Ncal(0,K)$ to be of Gaussian form. Then, we can compute the gradient
\begin{equation}\label{eq:sinkhorn_bary_proof0}
\begin{aligned}
    &\nabla_K\sum_{i=1}^N \lambda_i S_2^\epsilon\left(\Ncal(0,K), \Ncal(0,K_i)\right)\\
    =& \nabla_K\sum_{i=1}^N \lambda_i \Big(
    \OT^\epsilon_{d^2}\left(\Ncal(0,K), \Ncal(0,K_i)\right)\\
    &-\frac{1}{2}
    \OT^\epsilon_{d^2}\left(\Ncal(0,K), \Ncal(0,K)\right)\\
    &-\frac{1}{2}\OT^\epsilon_{d^2}\left(\Ncal(0,K_i), \Ncal(0,K_i)\right)\Big),
\end{aligned}
\end{equation}
where the last term disappears. Then, we can use the gradient of the first term, which we computed in \eqref{eq:barycenter_proof3}. A very similar computation yields
\begin{equation}\label{eq:sinkhorn_bary_proof1}
\begin{aligned}
    \nabla_K\OT^\epsilon_{d^2}\left(K,K\right)
    =  \frac{\epsilon}{2}K^{-\frac{1}{2}}\left(
    I + \frac{4}{\epsilon}K - \left(I + \frac{16}{\epsilon^2}K^2\right)^\frac{1}{2}
    \right)K^{-\frac{1}{2}}.
\end{aligned}
\end{equation}
Substituting \eqref{eq:barycenter_proof3} and \eqref{eq:sinkhorn_bary_proof1} into \eqref{eq:sinkhorn_bary_proof0} yields
\begin{equation}\label{eq:sinkhorn_bary_proof2}
\begin{aligned}
    &\nabla_K\sum_{i=1}^N \lambda_i S_2^\epsilon\left(\Ncal(0,K), \Ncal(0,K_i)\right)\\
    =& \frac{\epsilon}{4}\sum_{i=1}^N\lambda_i K^{-\frac{1}{2}}\left(
    \left(I + \frac{16}{\epsilon^2}K^2\right)^\frac{1}{2}
    - \left(I + \frac{16}{\epsilon^2}K^\frac{1}{2}K_iK^\frac{1}{2}\right)^\frac{1}{2}
    \right)K^{-\frac{1}{2}}.
\end{aligned}
\end{equation}
When \eqref{eq:sinkhorn_bary_proof2} is set to zero, we find, that the optimal $\bar{K}$ satisfies the relation given in~\eqref{eq:sinkhorn_barycenter}.
\qed
\end{proof}

\begin{figure}[h]
    \centering
    \includegraphics[width=\linewidth]{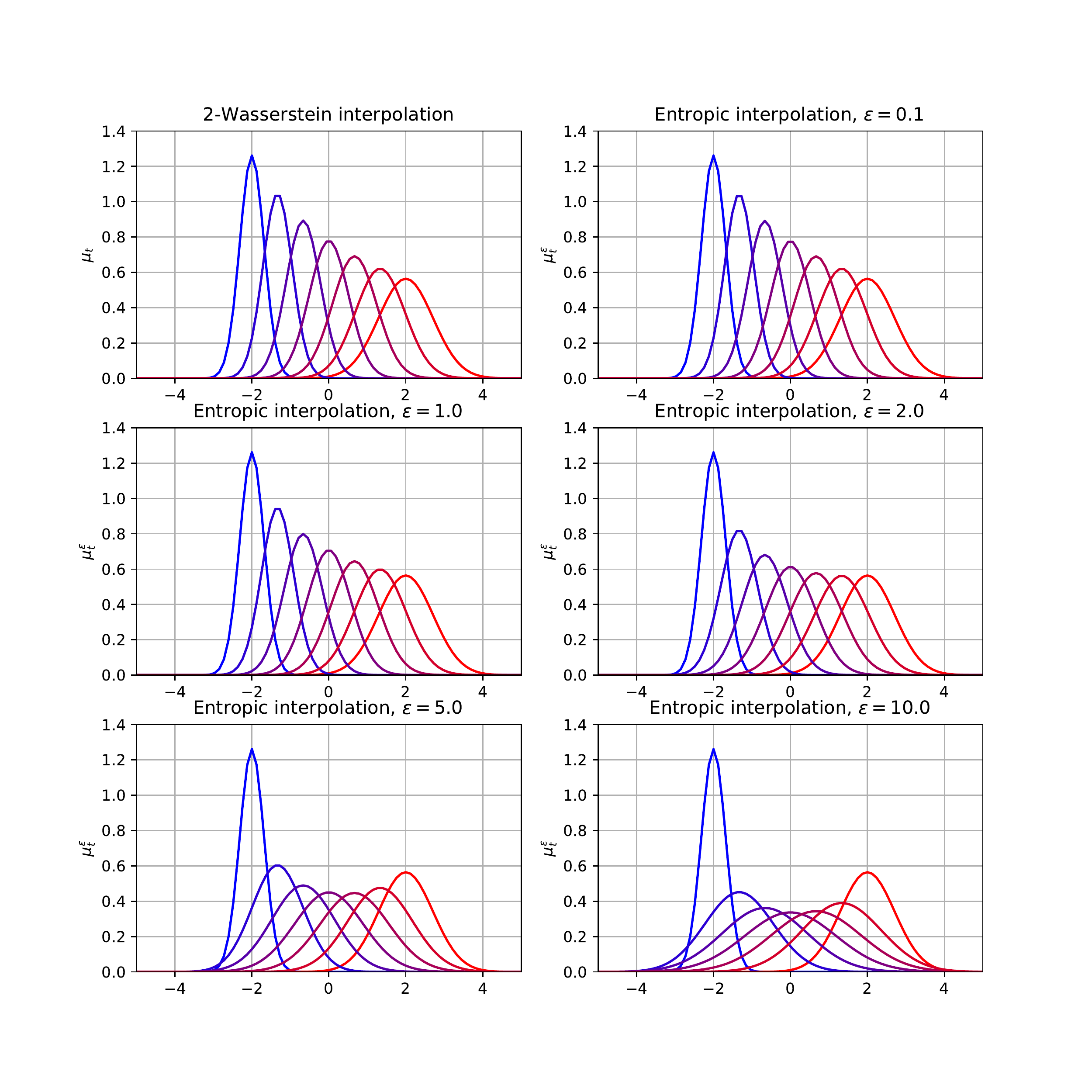}
    \caption{Entropic interpolants $\mu_t^\epsilon$ between two one-dimensional Gaussians given by $\mu_0 = \mathcal{N}(-2,0.1)$ (blue) and $\mu_1 = \mathcal{N}(2,0.5)$ (red), with varying regularization strengths $\epsilon$, accompanied by the $2$-Wasserstein interpolant in the top-left corner (corresponding to $\epsilon=0$).}
    \label{fig:1d_interpolant}
\end{figure}
\textbf{Fixed-point iteration.} The fixed-point iteration algorithm is defined by
\begin{equation}\label{eq:fixed_point}
    x_{k+1} = F(x_{k}),
\end{equation}
where the initial case $x_0$ is handpicked by the user. The \emph{Banach fixed-point theorem} is a well-known result stating that such an iteration converges to a fixed-point, i.e. an element $x$ satisfying $x = F(x)$, if $F$ is a \emph{contraction mapping}.

In the case of the $2$-Wasserstein barycenter given in~\eqref{eq:wasserstein_barycenter}, the fixed-point iteration can be shown to converge~\cite{alvarez16} to the unique barycenter. In the entropic $2$-Wasserstein and the $2$-Sinkhorn cases we leave such a proof as future work. However, while computing the numerical results in Section~\ref{sec:numerics}, the fixed-point iteration always succeeded to converge.

\section{Numerical Illustrations}\label{sec:numerics}
We will now illustrate the resulting entropic $2$-Wasserstein distance and $2$-Sinkhorn divergence for Gaussians by employing the closed-form solutions to visualize entropic interpolations between end point Gaussians. Furthermore, we emply the fixed-point iteration~\eqref{eq:fixed_point} in conjunction with the fixed-point expressions of the barycenters for their visualization.

First, we consider the interpolant between one-dimensional Gaussians given in Fig.~\ref{fig:1d_interpolant}, where the densities of the interpolants are plotted. As one can see, increasing $\epsilon$ causes the middle of the interpolation to flatten out. This results from the Fokker-Plank equation~\eqref{eq:fokker_plank}, which governs the diffusion of the evolution of processes that are objected to Brownian noise. In the limit $\epsilon \to \infty$, we would witness a heat death of the distribution.

The same can be seen in the three-dimensional case, depicted in Fig.~\ref{fig:3d_interpolant}, visualized using the code accompanying~\cite{feragen17}. Here, the ellipsoids are determined by the eigenvectors and -values of the covariance matrix of the corresponding Gaussian, and the colors visualize the level sets of the ellipsoids. Note that a large ellipsoid corresponds to high variance in each direction, and does not actually increase the mass of the distribution. Such visualizations are common in \emph{diffusion tensor imaging} (DTI), where the tensors (covariance matrices) define Gaussian diffusion of water at voxels images produced by magnetic resonance imaging (MRI)~\cite{arsigny06}.

Finally, we consider the entropic $2$-Wasserstein and Sinkhorn barycenters in Fig.~\ref{fig:3d_barycenters}. We consider four different Gaussians, placed in the corners of the square fields in the figure, and plot the barycenters for varying weights, resulting in the \emph{barycentric span} of the four Gaussians. As the results show, the barycenters are very similar under the two frameworks with small $\epsilon$. However, as $\epsilon$ is increased, the Sinkhorn barycenter seems to be more resiliant against the fattening of the barycenters, which can be seen in the $2$-Wasserstein case.

\begin{figure}[H]
    \centering
    \includegraphics[width=\linewidth]{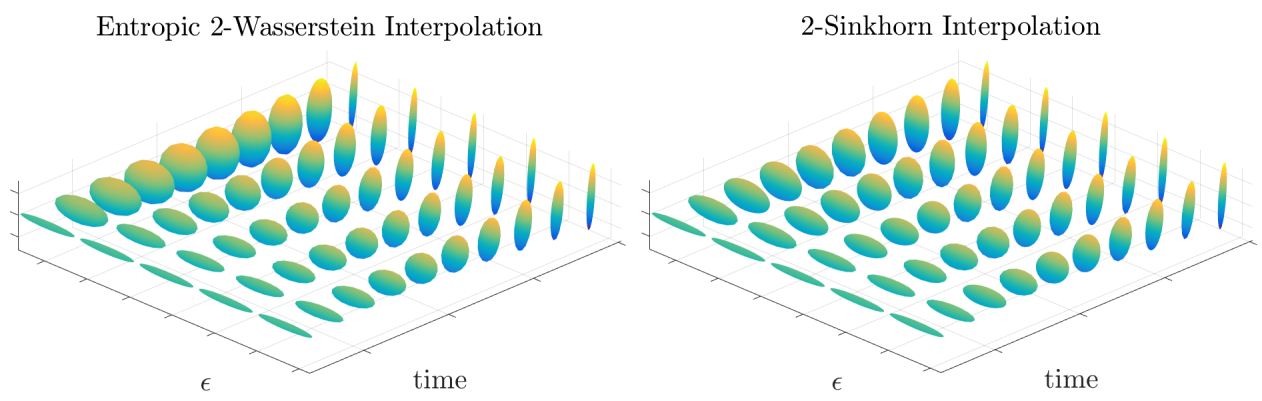}
    \caption{Interpolants between two three-dimensional Gaussians with varying regularization strengths $\epsilon$, accompanied by the $2$-Wasserstein interpolant, given by the first row (parallel to the time axis). The following rows visualize the interpolation for $\epsilon\in\{0.01, 1, 2, 5, 20\}$ in increasing order.}
    \label{fig:3d_interpolant}
\end{figure}

\begin{figure}[H]
    \centering
    \includegraphics[width=\linewidth]{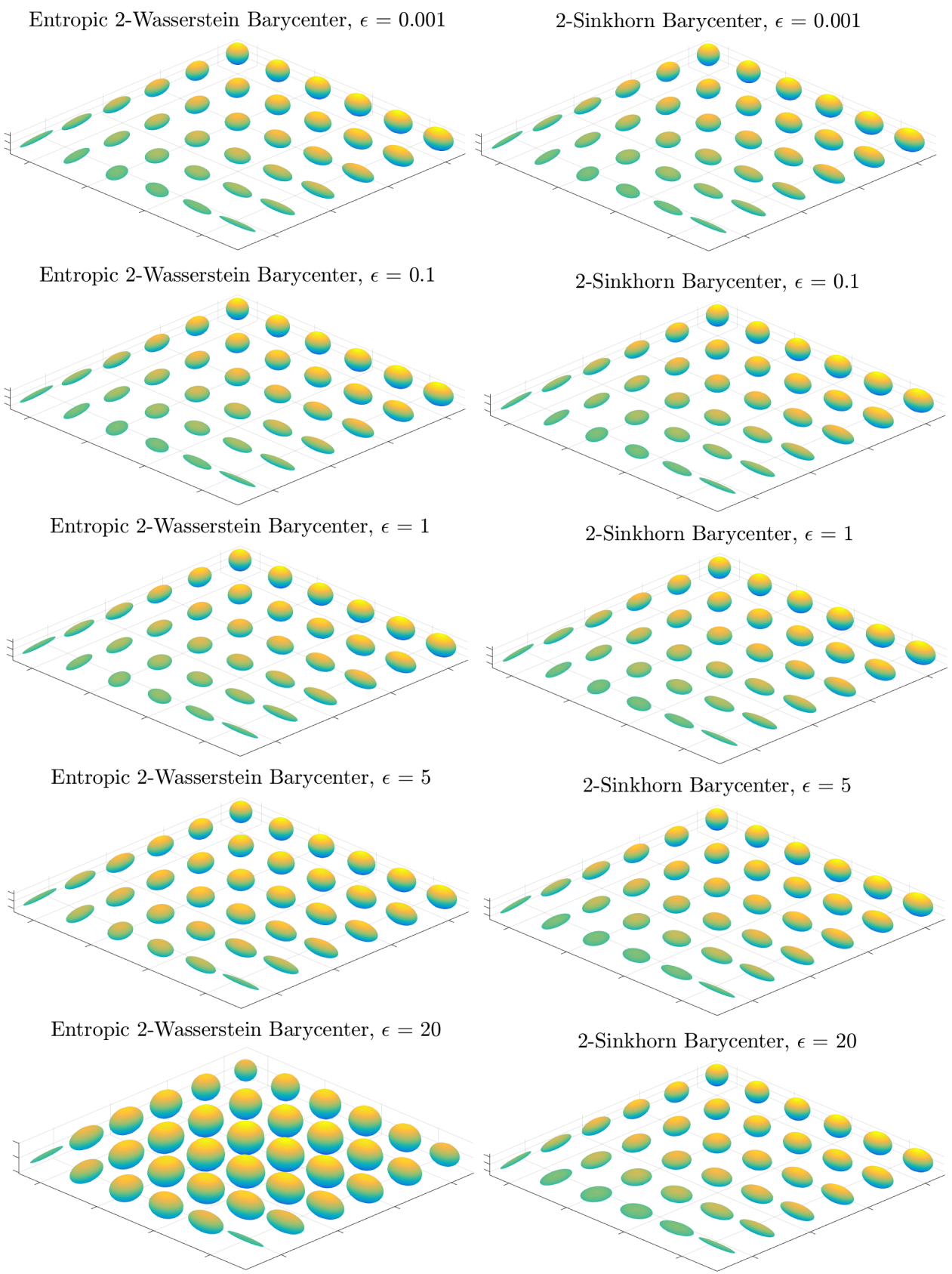}
    \caption{Barycentric spans of the four corner tensors under the entropic $2$-Wasserstein metric and the $2$-Sinkhorn divergence for varying $\epsilon$.}
    \label{fig:3d_barycenters}
\end{figure}

\section*{Acknowledgements}
This work was initiated during the authors' stay at the Institute for Pure and Applied Mathematics (IPAM), which is
supported by the National Science Foundation (Grant No. DMS-1440415). AM was supported by Centre for Stochastic Geometry and Advanced Bioimaging, funded by a grant from the Villum Foundation. AG acknowledges funding by the European Research Council under H2020/MSCA-IF ``OTmeetsDFT'' (grant no 795942).

\bibliographystyle{plain}
\bibliography{bib}

\section*{Appendix A: distributional solutions of Fokker-Planck equation}

We just recall the definition of distributional solution of the Fokker-Planck equation. \medskip

\begin{deff}
We say that a family of pairs measures/vector fields $(\eta_t,v_t)$ with $v_t \in L^1(\eta_t;\R^n)$ and $\int^1_0\Vert v_t\Vert_{L^1(\eta_t)}dt = \int^1_0\int_{\R^n}\vert v_t\vert d\eta_t dt$ solves the continuity equation on $]0,T[$ in the distributional sense if for any bounded and Lipschitz test function $f \in C^1_c(]0,T[\times \R^n)$
\[
\int^1_0\int_{\R^n}(\partial_tf)d\eta_tdt + \int^1_0\int_{\R^n}\left(\nabla f\cdot v_t-\frac{\ep}{2}\Delta f \right)d\eta_tdt = 0.
\]
\end{deff}

\section*{Appendix B: Alternative Proof of Theorem~\ref{thm:ent_was_gaussian} \textbf{b.}}\label{appendixB}
Recall, that by Propositions \ref{prop:restriction_to_centered} and \ref{prop:centered_gaussian_plan}, we can restrict to plans that are centered Gaussians, that is,

\begin{equation}
    \gamma = \N(0,\Gamma),\quad\Gamma 
    = \begin{bmatrix}
    K_1 & C^T \\
    C & K_2
    \end{bmatrix}.
    \label{eq:gaussian_gamma}
\end{equation}

Substituting \eqref{eq:gaussian_gamma} into \eqref{eq:mainKL} yields
\begin{equation}\label{eq:ent_was_gaussian2}
\begin{aligned}
    \OT_{d^2}^\epsilon(\mu_1, \mu_2) =& \min_{C\in \mathbb{R}^{n \times n}} F(C)\\
    :=& \min_{C\in \mathbb{R}^{n\times n}} \left\lbrace
    \Tr(K_1) + \Tr(K_2)\phantom{\frac{K}{K}}\right. \\
    &- 2 \left.\Tr(C) + \frac{\epsilon}{2}\log\left(
    \frac{\det(K_1 K_2)}{\det(\Gamma)}
    \right)
    \right\rbrace.
\end{aligned}
\end{equation}
The covariance matrix $\Gamma$ should be a symmetric positive-definite matrix, which is equivalent to its Schur complement $S(C)$ being positive definite, that is,
\begin{equation}
    S(C):= K_1 - C^T K_2^{-1} C \succeq 0. 
\end{equation}
If $S(C)$ fails to be strictly positive definite, $F(C)$ explodes to infinity, and so it suffices to consider $C$ so that
\begin{equation}
    S(C) \succ 0.
    \label{eq:schur_condition}
\end{equation}

Now recall the Schur block matrix determinant formula
\begin{equation}
    \det(\Gamma) = \det( S(C))\det(K_2).
    \label{eq:schur_determinant}
\end{equation}
Then, following the argumentation in the proof of~\cite[Prop. 7]{givens84}, when the value of $S(C)=S$ is fixed, we can write
\begin{equation}
    \max_{C~:~S(C) = S}\Tr(C) = \Tr\left(K_2^\frac{1}{2} (K_1 - S) K_2^\frac{1}{2}\right)^\frac{1}{2},
    \label{eq:tr_min_over_fiber}
\end{equation}
and so applying \eqref{eq:schur_determinant} and \eqref{eq:tr_min_over_fiber} to \eqref{eq:ent_was_gaussian2}, we get
\begin{equation}
\begin{aligned}
   \min_{C~:~S(C) = S} F(C) =& 
    \Tr(K_1) + \Tr(K_2) - 2 \Tr\left(K_2^\frac{1}{2} (K_1 - S) K_2^\frac{1}{2}\right)^\frac{1}{2}\\
    &+ \frac{\epsilon}{2}\left(
    \log\det(K_1) - \log\det(S)
    \right),
    \label{eq:ent_was_gaussian3}
\end{aligned}
\end{equation}
leaving us with the task of minimizing \eqref{eq:ent_was_gaussian3} with respect to $S$. Note that we could maximize \eqref{eq:tr_min_over_fiber} independently with respect to $C$, as $\det(\Gamma)$ is constant over the fiber $\{C: S(C) = S\}$. 

As F is strictly convex with respect to $S$, a solution to \eqref{eq:ent_was_gaussian2} can be found when the gradient of the expression with respect to $S$ is zero, leading to
\begin{equation}
    \nabla_S F(S) = K_2^\frac{1}{2} 
    \left(K_2^\frac{1}{2} \left(K_1 - S\right) K_2^\frac{1}{2} \right)^{-\frac{1}{2}}
    K_2^\frac{1}{2} - \frac{\epsilon}{2}S^{-1} = 0.
    \label{eq:ent_was_gaussian_gradient}
\end{equation}
Moving the second term to RHS, multiplying \eqref{eq:ent_was_gaussian_gradient} by $(K _1-S)^\frac{1}{2}$ from right, multiplying each side by their corresponding transposes, and some elementary manipulations of the equation, we arrive at a \emph{continuous algebraic Riccati equation} (CARE)
\begin{equation}
    \epsilon^2 K_1 - \epsilon^2 S - 4 S K_2 S = 0.
    \label{eq:riccati}
\end{equation}
In general, CAREs do not admit an analytical solution. However, we are in luck, as one can check that \eqref{eq:riccati} is solved by
\begin{equation}
    \hat{S} = \frac{\epsilon}{8} K_2^{-\frac{1}{2}} \left(
    -\epsilon I + \left(\epsilon^2 I + 16 K_2^{\frac{1}{2}} K_1 K_2^{\frac{1}{2}} \right)^\frac{1}{2}
    \right)K_2^{-\frac{1}{2}}.
\end{equation}

Finally, it is straight-forward to check that the solution $\hat{S}$ is indeed symmetric and positive-definite, and therefore satisfies \eqref{eq:schur_condition}. Plugging $\hat{S}$ in \eqref{eq:ent_was_gaussian3}, noticing that $K_2^\frac{1}{2}K_1K_2^\frac{1}{2}$ has same eigenvalues as $K_1K_2$, and some simplifications concludes the proof.

Now, we compute the OT quantity given $\hat{S}$. We first compute the trace term \eqref{eq:ent_was_gaussian3}, which gives
\begin{equation}
\begin{aligned}
    \Tr\left(
    K_2^\frac{1}{2}(K_1 - \hat{S}) K_2^\frac{1}{2}
    \right)^\frac{1}{2} &=
    \Tr\left(K_1K_2 - \frac{\epsilon}{8}\left(
    -\epsilon I + \left( \epsilon^2 I + 16K_1K_2\right)^\frac{1}{2}
    \right)\right)^\frac{1}{2}\\
    &= \Tr\left(
    \frac{\epsilon^2}{16}I +  K_1K_2 + \frac{\epsilon^2}{16}I
    - \frac{\epsilon^2}{8}\left(
    I + \frac{16}{\epsilon^2}K_1K_2
    \right)^\frac{1}{2}
    \right)^\frac{1}{2}\\
    &= \frac{\epsilon}{4}\Tr\left(
    \left(-I + \left(
    I + \frac{16}{\epsilon^2}K_1K_2
    \right)^\frac{1}{2} \right)^2
    \right)^\frac{1}{2}\\
    &= \frac{\epsilon}{4}\Tr\left(
    -I + \left( I + \frac{16}{\epsilon^2}K_1K_2\right)^\frac{1}{2}
    \right)\\
    &= \frac{\epsilon}{4}\left(\Tr\left(
    M_\epsilon
    \right) -2n
    \right)\\ 
\end{aligned}
\end{equation}
For the other term, write $\{\lambda_i\}_{i=1}^n$ for the eigenvalues of $K_1K_2$ and $m_i = 1 + \frac{16}{\epsilon^2}\lambda_i$
\begin{equation}
    \begin{aligned}
    \log\det(K_1) - \log\det(\hat{S})
    =& \log\det(K_1K_2)\\
    &- \log\det\left(\frac{\epsilon^2}{8}\left(
    -I + \left(
    I + \frac{16}{\epsilon^2}K_1K_2
    \right)^\frac{1}{2} 
    \right)
    \right)\\
    =& \sum_{i=1}^n \log \left(
    \frac{\epsilon^2(m_i - 1)}{16}
    \right)
    - \sum_{i=1}^n \log \left(
    \frac{\epsilon^2}{8}( m_i^\frac{1}{2} - 1)
    \right) \\
    =& \sum_{i=1}^n \log \left(
    \frac{1}{2}(1 + m_i^\frac{1}{2})
    \right)\\
    =& \sum_{i=1}^n\log\left(
    1 + \left(
    1 + \frac{16}{\epsilon^2}\lambda_i
    \right)^\frac{1}{2}
    \right) - n\log 2 \\
    =& \log\det(M_\epsilon) - n\log 2.
    \end{aligned}
\end{equation}

\end{document}